\documentclass{article}

\usepackage{microtype}
\usepackage{graphicx}
\usepackage{subcaption}
\usepackage{physics}
\usepackage{booktabs} %
\usepackage[most]{tcolorbox}
\usepackage{setspace}
\usepackage{hyperref}
\usepackage{url}

\usepackage[numbers]{natbib}
\usepackage[preprint]{neurips_2026}
\makeatletter
\renewcommand{\@notice}{}
\renewcommand{\@maketitle}{%
  \vbox{%
    \hsize\textwidth
    \linewidth\hsize
    \vskip 0.05in
    \@toptitlebar
    \centering
    {\LARGE\bf \@title\par}
    \@bottomtitlebar
    \vskip 6\p@ %
    \def\And{%
      \end{tabular}\hfil\linebreak[0]\hfil%
      \begin{tabular}[t]{c}\bf\rule{\z@}{18\p@}\ignorespaces%
    }
    \def\AND{%
      \end{tabular}\hfil\linebreak[4]\hfil%
      \begin{tabular}[t]{c}\bf\rule{\z@}{18\p@}\ignorespaces%
    }
    \begin{tabular}[t]{c}\bf\rule{\z@}{18\p@}\@author\end{tabular}%
    \vskip 0.12in \@minus 0.06in
  }
}
\makeatother

\usepackage{amsmath}
\usepackage{amssymb}
\usepackage{mathtools}
\usepackage{amsthm}

\usepackage[T1]{fontenc}
\usepackage{xcolor}
\usepackage{enumitem}
\usepackage{amsmath,amssymb,amsfonts}
\usepackage{subcaption}
\usepackage{tikz}
\usepackage{lipsum}
\usepackage{graphicx,scalerel}
\usepackage{amsthm}
\usepackage{thmtools,thm-restate}
\usepackage{nicefrac}

\usepackage{mdframed}
\usepackage{algorithm}
\usepackage{algorithmic}
\usepackage[most]{tcolorbox}
\usepackage{setspace}
\definecolor{encInk}{RGB}{166,124,52}  \definecolor{encBar}{RGB}{244,234,224}
\definecolor{encBody}{RGB}{252,249,246}
\definecolor{decInk}{RGB}{74,122,117}  \definecolor{decBar}{RGB}{226,237,235}
\definecolor{decBody}{RGB}{249,252,251}
\definecolor{cmt}{RGB}{140,140,140}
\newcommand{\cmt}[1]{\hfill{\color{cmt}// #1}}
\tcbset{algbox/.style n args={4}{
  enhanced, sharp corners, colframe=#1, colback=#2, boxrule=0.6pt,
  left=2mm, right=2mm, top=0.5mm, bottom=0.5mm, width=\linewidth,
  fontupper=\scriptsize,
  title=#4, fonttitle=\scriptsize\bfseries\scshape, coltitle=#1, colbacktitle=#3,
  attach boxed title to top left={xshift=0pt, yshift=-0.6pt},
  boxed title style={sharp corners, boxrule=0.6pt, colframe=#1}, titlerule=0pt }}
\newcommand{\algstyle}{\renewcommand{\algorithmicrequire}{\textit{\textbf{Input:}}}}
\usepackage{float} %
\usepackage{wrapfig}
\theoremstyle{plain}
\newtheorem{theorem}{Theorem}[section]
\newtheorem{proposition}[theorem]{Proposition}
\newtheorem{lemma}[theorem]{Lemma}

\theoremstyle{definition}

\theoremstyle{remark}

\usepackage{xcolor}

\definecolor{amber}{RGB}{206,18,86}
\definecolor{blueish}{RGB}{43,140,190}
\definecolor{pu}{HTML}{984ea3}

\usepackage[textsize=tiny]{todonotes}

\title{Requential Coding: Pushing the Limits of Model Compression with Self-Generated Training Data}

\author{%
  Shikai Qiu \\
  New York University \\
  \And
  Marc Finzi \\
  Carnegie Mellon University \\
  \And
  Yujia Zheng \\
  Carnegie Mellon University \\
  \And
  Kun Zhang \\
  Carnegie Mellon University \\
  \And
  Andrew Gordon Wilson \\
  New York University \\
}

\begin{document}

\maketitle

\begin{abstract}

Compression is fundamental to intelligence. A model that can represent its training data as a short code has discovered regularities that enable generalization. Large neural networks may learn functions far \emph{simpler} than their parameter counts suggest, but it is challenging to construct codes that realize this simplicity. Parameter-based methods such as quantization produce code lengths that scale with model size, insensitive to how much information the parameters store. Prequential coding bypasses this issue by compressing the training trajectory, but codes the exact data sequence regardless of how much the model learns, yielding large codes when the data has high entropy. We introduce \emph{requential coding}\renewcommand{\thefootnote}{$\dagger$}\footnote{Code available at \url{https://github.com/shikaiqiu/requential-coding}.}\setcounter{footnote}{0}\renewcommand{\thefootnote}{\arabic{footnote}}, where a teacher model selects training samples drawn from the student's own distribution. The student's code records only these selections, which cost bits only where teacher and student disagree. The resulting code length is independent of parameter count and data entropy, and often orders of magnitude shorter than the prequential counterpart, with an advantage that grows with scale. This compression sheds light on phenomena inaccessible to prior compressors. Holding loss fixed, larger models and ensembles compress to much smaller sizes \emph{despite} more parameters. Plugged into a PAC-Bayes bound, the requential code yields state-of-the-art generalization guarantees for billion-parameter LLMs, outperforming bounds built on aggressive post-training quantization even granted zero error. The bound \emph{tightens with scale} in the compute-optimal regime, as models become increasingly compressible relative to dataset size. The same code predicts that models gradually overfit when trained for multiple epochs. It also isolates the learnable information in a dataset from its unpredictable, random content, revealing that lower-entropy text holds far more learnable structure than higher-entropy image data.

\end{abstract}

\section{Introduction}

Measuring compression is key to understanding generalization in deep learning.
In order to compress data, a model must discover regularities that facilitate generalization. This intuition underlies fundamental principles of induction, such as \emph{Occam's razor}: the simplest explanation consistent with observations is most likely to be true. A strong enough compression can guarantee a model's generalization performance, limit memorization, and even reveal how much learnable information content is in the training data. Indeed, a growing body of evidence suggests that neural networks often learn functions far simpler than their parameters could express \citep{li2018measuring, frankle2018lottery, zhou2018non, lotfi2022pac, wilson2025deep}.

However, finding a sufficiently good compression at scale remains a fundamental open question. It could be that larger neural networks find even simpler, more compressible functions, but demonstrating this compressibility becomes increasingly difficult with scale. As we scale model and data size, existing model compression schemes are inflated by quantities unrelated to actual learning: post-training quantization \citep{frantar2023gptq, tseng2024quip}, which directly compresses the learned parameters, produces codes that increase linearly with model size regardless of information stored. Alternatively, prequential coding \citep{dawid1984present, blier2018description} codes a model through its training data, compressed using the training process itself, but the code grows linearly with dataset size as it must encode the exact dataset encountered regardless of how much information the model extracts. Neither approach captures how information actually transfers from data to the model. Accordingly, the complexity estimates they produce clash both with the empirical fact that scaling models and data improves generalization, and with the theory of infinite limits, where networks converge to well-defined limits as size grows \citep{yang2022tensor,yang2023feature,dey2026don}.

\begin{figure*}[t]
\centering
\includegraphics[width=\linewidth]{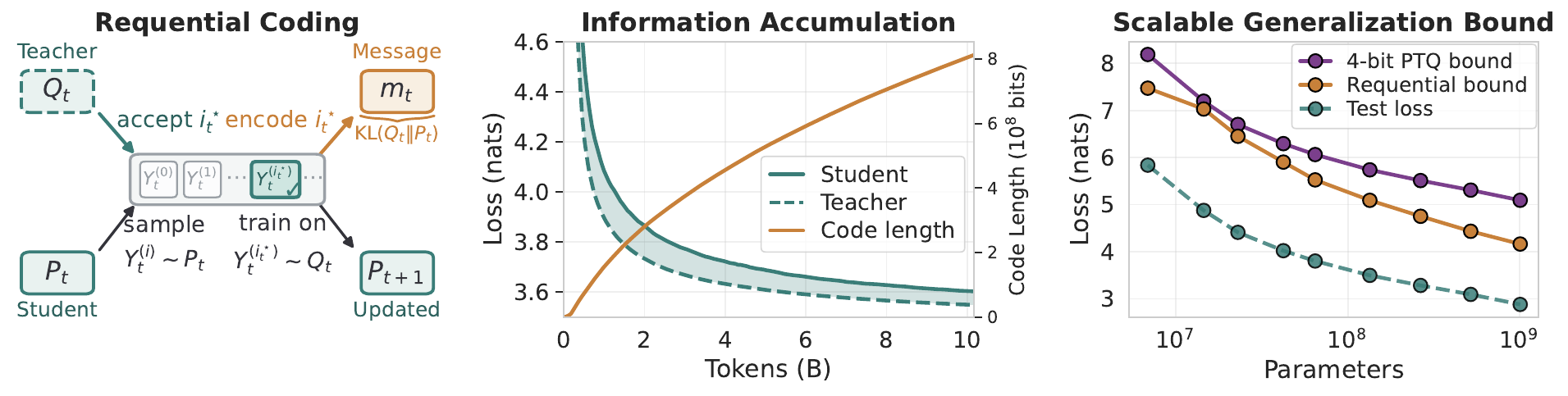}
\caption{\small\textbf{Requential coding achieves strong model compression.} (\textbf{Left}) 
The student model $P_t$ being compressed samples candidates $Y_t^{(0)}, Y_t^{(1)}, \ldots \overset{\mathrm{i.i.d.}}{\sim} P_t$ for its own training data. A teacher model $Q_t$ accepts an index $i_t^\star$ chosen so that $X_t = Y_t^{(i_t^\star)}$ is marginally distributed as $Q_t$, and the student trains on $X_t$ to yield $P_{t+1}$. A message $m_t$ of about $\mathrm{KL}(Q_t\|P_t)$ bits encodes the accepted index using relative entropy coding (REC) and is appended to the student's code, and the process repeats. The student can be decoded from the messages alone, without the teacher. (\textbf{Middle}) As training progresses, shown here for language modeling on FineWeb, the student tracks the teacher's loss on real data with the loss gap approximately equal to their KL. The code length for the student, approximately the cumulative teacher-student KL, is roughly equal to the integral of their loss gap. (\textbf{Right}) Strong model compression sheds lights on many phenomena, such as generalization. Our code yields state-of-the-art generalization bounds for compute-optimal LLMs, tightening with scale and outperforming the lossless idealization of 4-bit post-training quantization that set the previous best bounds.}
\label{fig:teaser}
\end{figure*}

We introduce \emph{requential coding}, a substantially more efficient model compressor that sheds light on a variety of generalization phenomena, and enables state-of-the-art generalization bounds, for large neural networks. Requential coding is fundamentally based on two observations. First, like prequential coding, we should compress the data instead of parameters to leverage the sample efficiency of the model. Second, much of the information in the training data pins down microscopic details of the model that are unimportant for performance. Just as quantization and pruning discard unimportant bits in the parameters, we can discard unimportant bits in the data, coding an approximate model trained on a surrogate dataset that is vastly cheaper to describe. For example, instead of coding a particular realization of a training batch, we can code a random batch from the training distribution. Moreover, we only need to specify how that distribution departs from what the model already knows.

Requential coding works by coding a student model $P_t$, a generative model trained iteratively on data it itself generates, with its training samples chosen by a stronger teacher model $Q_t$. At each step $t$, the student proposes candidate samples, $Y_t^{(0)}, Y_t^{(1)}, \ldots$, drawn from its own distribution using a pseudorandom number generator with a known seed, so the code needs to record only the index of the proposal the teacher accepts. With an appropriate acceptance rule, relative entropy coding (REC) \citep{flamich2020compressing,theis2022algorithms} chooses this index so that the accepted $X_t$ is marginally a draw from the teacher $Q_t$, and coded in roughly $\mathrm{KL}(Q_t\|P_t)$ bits.
The teacher can evolve arbitrarily, typically by training on real data, and need not itself be transmitted.  Figure~\ref{fig:teaser} (left) illustrates this process.
Equivalently, the student is a lossy compression of the teacher, obtained by distillation with the teacher samples transmitted efficiently via REC.
As training progresses, the student closely tracks the teacher and its code length accumulates gradually, as shown in Figure~\ref{fig:teaser} (middle).

Requential coding is a strong model compressor, achieving orders of magnitude shorter codes than prequential coding (Section~\ref{sec:experiments}) and a rate near 1 bit per parameter for compute-optimal LLMs (Section~\ref{sec:bounds}). Moreover, this strong compression reveals a range of phenomena beyond the reach of parameter-based codes. It shows larger models and bigger ensembles can be more compressible despite having more parameters (Section~\ref{sec:scale}). It delivers state-of-the-art generalization bounds for billion-parameter compute-optimal LLMs with a generalization gap that decays as a power law in model size (Section~\ref{sec:bounds}). The same code further predicts models gradually overfit under data repetition (Section~\ref{sec:overfitting}), and measures how much useful information can be extracted from a dataset, isolating learnable structure from random information for principled data selection (Section~\ref{sec:info}).

\section{Background}

 We now review existing methods for model compression, clarifying their inefficiencies along the way, and then relative entropy coding, the core primitive behind requential coding. Throughout this paper, $\log$ denotes $\log_2$ and KL divergence is measured in bits.

\paragraph{Compressing Neural Network Parameters.}
Neural network parameters often contain substantial redundancy, allowing them to be compressed to far fewer degrees of freedom than their raw counts without significant loss in performance. Pruning and sparsification remove parameters that contribute little to predictive performance, revealing that often only a small fraction are functionally necessary \citep{lecun1989optimal,hassibi1993optimal,han2015deep, frankle2018lottery, lee2018snip, gale2019state}. Similarly, low-rank and subspace training methods constrain optimization to a small set of directions, producing models whose parameters can be compressed with matrix factorizations \citep{li2018measuring,lotfi2022pac,hu2022lora}. Instead of removing parameters, post-training quantization (PTQ) reduces the precision of the trained parameters far below the 32-bit or 16-bit floating point formats used during training \citep{han2015deep,nagel2020up,dettmers2022gpt3,frantar2023gptq,lin2024awq}. 
Among these methods, PTQ achieves state-of-the-art compression for large language models (LLMs) as judged by performance per bit and is widely adopted to reduce inference costs, with the best methods reaching $\leq 4$ bits per parameter at minor loss in performance via Hessian-based adaptive strategies \citep{frantar2023gptq,tseng2024quip}. 

A fundamental limitation of the above methods is that they aim to directly compress the final parameter values with little regard to how much information the parameters actually encode, failing to fully decouple the parameter count of the model from its compressed size. A model with billions of parameters trained only on a few data points after random initialization must be highly compressible, yet its parameters are typically neither sparse nor low-rank, and there is a limit on how much each parameter can be quantized without significantly altering the model output. As modern models are typically trained far short of the information capacity of their parameters \citep{kaplan2020scaling, hoffmann2022training}, parameter-based methods fail to reach the true limits of compression.

\paragraph{Prequential Coding.} \label{sec:preq}
Instead of compressing the parameters, prequential coding \citep{dawid1984present, rissanen2003universal} compresses a dataset presented as a sequence of batches $X_0,\dots,X_{T-1},$ using a generative model trained sequentially on that dataset. The encoder and decoder start from a shared model $P_0$ and agree on an update rule $G$ (e.g. gradient descent). At each step $t \geq 0,$ the encoder encodes the next batch $X_t$ with $\log 1/ P_t(X_t)$ bits using a streaming entropy code (e.g., arithmetic coding), then trains the model on $X_t$ to yield $P_{t+1} = G(P_t, X_t)$. The decoder recovers $X_t$ from the encoded message and $P_t,$ already available by induction, performs the same update to obtain an identical $P_{t+1},$ and the process repeats. As the model's approximation of the true data distribution improves, it takes fewer bits to encode future data points. The total code length for the dataset is $
L_{\mathrm{preq}}(X_{0:T-1})=\sum_{t=0}^{T-1}\log 1/ P_t(X_t),$
the area under the training loss curve. As the decoder recovers the trained models $P_1, \ldots, P_T$, the code is a compression of the model $P_T$ as well.

While the prequential code no longer pays for the parameter count, it must losslessly compress the exact training data sequence regardless of how much information the model extracts from it. Consider the following decomposition of the code length in expectation over the data distribution $P^\star:$
\begin{equation}
\small
\mathbb{E}[L_{\mathrm{preq}}(X_{0:T-1})] = \underbrace{\sum_{t=0}^{T-1} H(X_t)}_{\text{data entropy}} + \underbrace{\sum_{t=0}^{T-1} \mathbb{E}[\mathrm{KL}(P^\star \,\|\, P_t)]}_{\text{approximation error}}.
\end{equation}
The first term is the irreducible entropy of the data source, which is paid even by a perfect predictor, accumulating at a linear rate even after the model stops learning from the additional, unpredictable data. The second term captures the gap between the model's predictions and the true distribution, which starts high and decreases as the model improves. An ideal code should not need to pay for either term in full. The data entropy should not be necessary if we are agnostic to which \emph{specific} sample $X_t$ is from $P^\star$ and instead have it be a \emph{random} sample. Paying the approximation error in full is likewise excessive because actual learning is incremental: the model can meaningfully absorb only a small improvement per step, empirically learning the simplest structures in the data first before moving to more complex patterns \citep{kalimeris2019sgd,liu2021probing,saphra2019understanding}, yet the prequential code pays the full remaining gap to the truth at every step. As we will soon see, requential coding addresses both issues.

Since prequential coding simultaneously encodes both the model and its training data, a commonly used heuristic for isolating the information stored in the final model $P_T$ alone is $L_{\mathrm{heuristic}}(P_T)=\sum_{t=0}^{T-1}\log 1/P_t(X_t) - \log 1/P_T(X_t),$ i.e., subtracting the compressed size of the data $X_{0:T-1}$ given $P_T$ from the combined code length for both $P_T$ and $X_{0:T-1}$ \citep{blier2018description,feldman1998information,whitney2020evaluating,zhang2020measuring,finzi2025compute,finzi2026entropy}. Unlike prequential coding, this heuristic only provides a non-rigorous estimate of the compressed model size but does not provide a valid compression and decompression scheme.

\paragraph{Relative Entropy Coding.}\label{sec:rec}
Relative entropy coding (REC) \citep{flamich2020compressing,theis2022algorithms} provides a compression primitive to transmit a random sample from a target distribution $Q$ using fewer bits than its entropy, by sampling candidates from a reference distribution $P$ (thought of as an approximation of $Q$) and selectively accepting them. We assume $P$ and $Q$ are discrete distributions, though the algorithm extends to continuous variables.  The encoder has access to both $P$ and $Q$ and can evaluate their likelihoods, while the decoder has only $P$. Both hold shared randomness $S$, which can be implemented by a pseudorandom number generator (PRNG) with a common seed. This randomness defines an indexed proposal sequence $Y_0,Y_1,\ldots \overset{\mathrm{i.i.d.}}{\sim} P.$ Because both sides can generate the same proposals from $P$, communicating a sample distributed as $Q$ reduces to communicating which proposal to accept: the REC procedure selects a proposal index $i^\star$ via an acceptance rule such that $Y_{i^\star}$ is marginally distributed as $Q$, and transmits a prefix-free code $m$ for that index. The decoder runs $\mathrm{REC.Decode}(P, m, S) \to Y_{i^\star},$ recovering $i^\star$ from $m$ and regenerating the accepted proposal $Y_{i^\star}$ using the shared randomness.  With a counter-based PRNG, $Y_{i^\star}$ can be regenerated directly from its index without generating earlier proposals. Figure~\ref{alg:requential-algs} (bottom) gives pseudocode for the simplest, inefficient implementation of REC using rejection sampling. More efficient approaches like \citet{li2018strong} achieve an expected code length bounded by $\mathrm{KL}(Q\|P) + \log(1+\mathrm{KL}(Q\|P)) + 5$ using the Poisson functional representation (PFR), approaching the information-theoretic lower bound $\mathrm{KL}(Q\|P),$ but the encoder needs to draw an unbounded number of proposals. Ordered random coding (ORC) \citep{theis2022algorithms} implements approximate sampling from $Q$ with the same code length bound while drawing on the order of $2^{\mathrm{KL}(Q\|P)}$ proposals.

The key to REC is that it reduces communicating a \emph{specific} sample from $Q$ to communicating only \emph{some} random sample whose marginal distribution is $Q$, by choosing among draws from $P.$ By being agnostic to \emph{which} sample is recovered, REC can spend far fewer bits, analogous to bits-back coding \citep{hinton1993autoencoders}. As a limiting case, when $P=Q$ the encoder can simply transmit a constant message and the decoder takes the first proposal from $P$ using the shared randomness, so only $O(1)$ bits are communicated rather than the $H(Q)$ bits entropy coding would require. More generally, when $Q$ is close to $P$ the expected message length can be much smaller than the naive $H(Q)$.

\section{Requential Coding}
Motivated by the shortcomings of both parameter-based compression and prequential coding, we introduce \emph{requential coding}, a highly efficient compression scheme for generative models whose code length depends on neither parameter count nor data entropy. We define the encoder and decoder protocols with the resulting code length and runtime (Section~\ref{sec:req_construction}), then evaluate requential coding against prequential coding and quantization on transformers trained on text and images (Section~\ref{sec:experiments}).

\renewcommand{\algorithmicrequire}{\textbf{Input:}}
\renewcommand{\algorithmicensure}{\textbf{Output:}}

\begin{figure}[t]
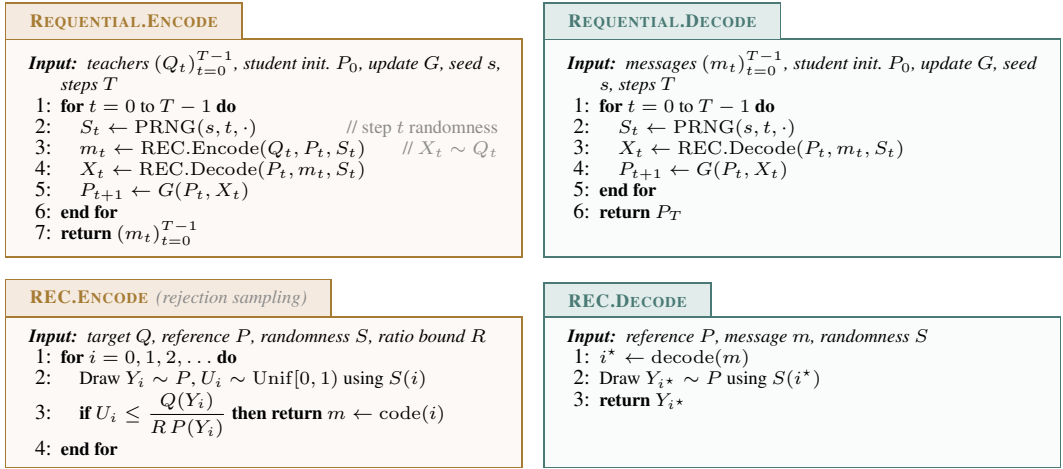

\noindent
\begin{minipage}[t]{0.49\linewidth}
\begin{tcolorbox}[algbox={encInk}{encBody}{encBar}{Requential.Encode}, equal height group=algr1]
\begin{algorithmic}[1]
\algstyle
\REQUIRE \textit{teachers $(Q_t)_{t=0}^{T-1}$, student init.\ $P_0$, update $G$, seed $s$, steps $T$}
\FOR{$t = 0$ to $T-1$}
  \STATE $S_t \leftarrow \mathrm{PRNG}(s,t,\cdot)$ \cmt{step $t$ randomness}
  \STATE $m_t \leftarrow \mathrm{REC.Encode}(Q_t, P_t, S_t)$ \cmt{$X_t \sim Q_t$}
  \STATE $X_t \leftarrow \mathrm{REC.Decode}(P_t, m_t, S_t)$
  \STATE $P_{t+1} \leftarrow G(P_t, X_t)$
\ENDFOR
\STATE \textbf{return} $(m_t)_{t=0}^{T-1}$
\end{algorithmic}
\end{tcolorbox}
\end{minipage}\hfill
\begin{minipage}[t]{0.49\linewidth}
\begin{tcolorbox}[algbox={decInk}{decBody}{decBar}{Requential.Decode}, equal height group=algr1]
\begin{algorithmic}[1]
\algstyle
\REQUIRE \textit{messages $(m_t)_{t=0}^{T-1}$, student init.\ $P_0$, update $G$, seed $s$, steps $T$}
\FOR{$t = 0$ to $T-1$}
  \STATE $S_t \leftarrow \mathrm{PRNG}(s,t,\cdot)$
  \STATE $X_t \leftarrow \mathrm{REC.Decode}(P_t, m_t, S_t)$
  \STATE $P_{t+1} \leftarrow G(P_t, X_t)$
\ENDFOR
\STATE \textbf{return} $P_T$
\end{algorithmic}
\end{tcolorbox}
\end{minipage}

\vspace{2.5mm}
\noindent
\begin{minipage}[t]{0.49\linewidth}
\begin{tcolorbox}[algbox={encInk}{encBody}{encBar}{REC.Encode~\,\normalfont\itshape\scriptsize\color{cmt}(rejection sampling)}, equal height group=algr2]
\begin{algorithmic}[1]
\algstyle
\REQUIRE \textit{target $Q$, reference $P$, randomness $S$, ratio bound $R$}
\FOR{$i = 0, 1, 2, \ldots$}
  \STATE Draw $Y_i \sim P$, $U_i \sim \mathrm{Unif}[0,1)$ using $S(i)$
  \STATE \textbf{if} $U_i \leq \dfrac{Q(Y_i)}{R\,P(Y_i)}$ \textbf{then return} $m \leftarrow \mathrm{code}(i)$
\ENDFOR
\end{algorithmic}
\end{tcolorbox}
\end{minipage}\hfill
\begin{minipage}[t]{0.49\linewidth}
\begin{tcolorbox}[algbox={decInk}{decBody}{decBar}{REC.Decode}, equal height group=algr2]
\begin{algorithmic}[1]
\algstyle
\REQUIRE \textit{reference $P$, message $m$, randomness $S$}
\STATE $i^\star \leftarrow \mathrm{decode}(m)$
\STATE Draw $Y_{i^\star} \sim P$ using $S(i^\star)$
\STATE \textbf{return} $Y_{i^\star}$
\end{algorithmic}
\end{tcolorbox}
\end{minipage}

\vspace{1.5mm}
\caption{\small\textbf{Requential coding.} (\textbf{Top})
The requential encoder uses a sequence of teachers to define the student's training distribution, whose samples are coded relative to the student via relative entropy coding (REC). The decoder reconstructs the identical training steps without needing the teachers. The encoder and decoder stay synchronized and the student state is never transmitted.
 (\textbf{Bottom}) To provide intuition, we show how REC can be implemented with rejection sampling. Samples are drawn from $P$ until one is accepted according to $Q$ using a bound $R \geq \max_x Q(x)/P(x)$, and the accepted index $i$ is transmitted via $\mathrm{code}(\cdot)$, a universal integer code inverted by $\mathrm{decode}(\cdot)$. From the index, the sample can be decoded using only $P$ by advancing the PRNG state. With an improved REC implementation, only $\mathrm{KL}(Q\|P)$ bits are needed on average to encode the index.
}
\label{alg:requential-algs}
\end{figure}

\subsection{Method}\label{sec:req_construction}
At a high level, requential coding changes prequential coding (Section~\ref{sec:preq}) in one way: rather than training on a pre-existing dataset, the student trains on data it itself generates, and the code records only the small amount of information a stronger teacher model contributes by deciding which of the self-generated samples are worth training on. Both the student and the teacher are generative models. We assume the encoder and the decoder can sample from the student, and that the encoder can additionally evaluate the likelihoods of both models. At each step, the student $P_t$ being coded generates candidate samples from its own distribution, each sample a batch of data, and the teacher $Q_t$ accepts one of them, $X_t$, under an acceptance rule that makes $X_t$ marginally a sample from the teacher. Training on $X_t$ is therefore distillation from $Q_t$, with the teacher's samples conveyed at a cost far below their entropy.

Specifically, the encoder and decoder agree on the student initialization $P_0$, update rule $G$ (e.g. gradient descent), PRNG seed $s$, number of training steps, and batch size. For the REC call at step $t$, both sides use shared randomness deterministically derived from $s$ and $t$ to define the same indexed proposal sequence $Y_t^{(0)},Y_t^{(1)},\ldots \overset{\mathrm{i.i.d.}}{\sim} P_t,$ which can be implemented with a counter-based PRNG keyed by $(s,t,i)$, so the $i$-th proposal $Y_t^{(i)}$ can be regenerated directly without generating proposals $Y_t^{(0)},\ldots,Y_t^{(i-1)}$. The encoder, which has access to $Q_t$, uses REC to choose an accepted proposal index $i_t^\star$ such that $Y_t^{(i_t^\star)}$ is marginally distributed as $Q_t$, and transmits a prefix-free code $m_t$ for that index. The decoder recovers $i_t^\star$ from $m_t$, regenerates the corresponding proposal from the shared seed, and sets $X_t = Y_t^{(i_t^\star)}$.
Both sides then apply the same update
$P_{t+1}=G(P_t,X_t),$
so their copies of the student remain synchronized. The teacher models $(Q_t)_t$ can be arbitrary, typically obtained by training on a stream of real data, and are needed only on the encoder side and never transmitted. The procedure is summarized in Figure~\ref{alg:requential-algs} and illustrated in Figure~\ref{fig:teaser} (left).

\paragraph{Code Length.} The code for the final student $P_T$ is the concatenation of messages $(m_0, \ldots, m_{T-1})$. Let $\ell_t=|m_t|$, and let $\mathcal F_{t-1}$ denote the history before REC call $t$, including $P_t$ and $Q_t$ but not the next message $m_t$ or sample $X_t$. The cumulative conditional expected code length satisfies
\begin{equation}\label{eq:req_codelen}
\overline{L}_{\mathrm{req}}
:=
\sum_{t=0}^{T-1}\mathbb E[\ell_t\mid\mathcal F_{t-1}]
\le
\sum_{t=0}^{T-1}
\left[
\mathrm{KL}(Q_t\|P_t)
+
2\log(1+\mathrm{KL}(Q_t\|P_t))
+\kappa
\right]
=:
\widehat{L}_{\mathrm{req}},
\end{equation}
with $\kappa<5.21$, which we prove in Appendix~\ref{appx:compute_overhead}. The extra logarithmic term relative to the familiar REC bound $\mathrm{KL}+\log(1+\mathrm{KL})+O(1)$ comes from using a universal integer code for the selected index, rather than a Zipf code tuned to $\mathrm{KL}(Q_t\|P_t),$ which the decoder cannot access. We show in Appendix~\ref{appx:realized-codelength} that the realized code length concentrates tightly around $\overline{L}_{\mathrm{req}}$ under typical training setups, and thus report the computable bound $\widehat{L}_{\mathrm{req}}$ as the code length in all experiments. Furthermore, the logarithmic and constant terms are negligible compared to the linear-scaling KL term for large batch sizes (typically $\gtrsim 1$M tokens for language models), so in practice $\widehat{L}_{\mathrm{req}}$ reduces to the cumulative teacher-student KL.

\paragraph{Runtime.} In most scientific applications we care only about evaluating the compressed model size in Eq.~\eqref{eq:req_codelen} rather than actually transmitting the model. In this case it suffices to run an equivalent stochastic process in which REC encoding and decoding are replaced by sampling $X_t$ directly from the teacher $Q_t$. We use this procedure to evaluate the requential code length throughout the paper. Suppose the teacher shares the student's architecture and advances by training on real data with the same batch size. Evaluating the code length then takes roughly $2\times$ the memory and $2.33\times$ the FLOPs of ordinary training: at each step we run one teacher forward pass to sample $X_t$, one student forward-backward pass on $X_t$ to advance $P_t$, and one teacher forward-backward pass to advance $Q_t$, contributing FLOPs in the ratio $\frac{1}{3}:1:1$. If the teacher checkpoints are already available, the compute overhead drops to a moderate $0.33\times$.

Actually transmitting a model can be prohibitively slow to encode, depending on the implementation.
For example, ORC draw about $2^{\mathrm{KL}(Q_t\|P_t)}$ proposals per call. 
When encoding time matters, we can accept a longer code in exchange for a shorter encoding time by dividing each batch into smaller blocks and transmitting one block at a time (see Figure~\ref{fig:block-size-tradeoff}). The decoding cost, in contrast, is unaffected by the block size, since the decoder only generates the accepted proposal using the decoded index, and requires FLOPs close to ordinary training. See Appendix~\ref{appx:compute_overhead} for further details.

\paragraph{How to Choose the Teacher.}
\begin{wrapfigure}{r}{0.32\linewidth}
\centering
\vspace{-1.1em}
\includegraphics[width=\linewidth]{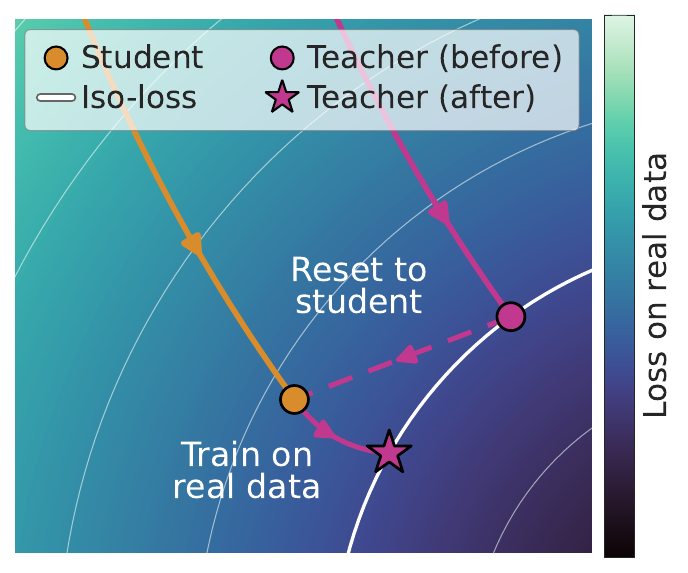}
\caption{\small\textbf{Iso-loss projection.} Teacher moves closer to the student while maintaining performance.}
\label{fig:isoloss-projection}
\vspace{-1.2em}
\end{wrapfigure}
A good teacher sequence balances two competing desiderata: each $Q_t$ should stay close to the current student, since every bit of divergence is charged to the code, yet it must run far enough ahead of the student to keep pulling the student toward the target distribution. We adopt the simplest possible choice where the teacher is trained on real data batches and shares the same architecture and hyperparameters as the student, so their divergence stays low due to similar training dynamics. We then introduce two improvements. 1) \emph{Teacher smoothing}: generating synthetic data from an exponential moving average (EMA) of the raw teacher checkpoints, reducing noise in the teacher trajectory that the student would otherwise pay to track.
2) \emph{Iso-loss projection}: periodically resetting the teacher to the current student and briefly training the teacher on real data with the student paused until the pre-reset loss is recovered, which approximately moves the teacher to the projection of the student onto its iso-loss surface (Figure~\ref{fig:isoloss-projection}). The new teacher has the same performance but typically a lower $\mathrm{KL}(Q_t\|P_t),$ shortening the subsequent code. We give pseudocode for both techniques in Appendix~\ref{app:smoothing-projection}, and expect substantial further gains from optimizing the teacher sequence.

\paragraph{Student as a Lossy Compression of the Teacher.}
Requential coding provides a lossless compression for a model trained via distillation\footnote{Specifically, \emph{hard} distillation where only the samples are used, not the logits.}, a technique widely used in the supervised setting \citep{hinton2015distilling} and increasingly central to modern LLM pre-training \citep{dubey2024llama, gemma2024, meta2025llama4} and post-training \citep{yang2025qwen3}. However, when the goal is to compress a model not itself trained via distillation, we can still leverage requential coding by distilling that model into a student, which is then a lossy compression of the original. We emphasize that there is no fundamental distinction between this type of lossy compression and those done by conventional parameter-based methods such as pruning and quantization: all trade losses in predictive performance for shorter codes and are lossy only insofar as the compressed model diverges from the original, for example as measured by their KL.

\begin{figure*}[t]
\centering
\includegraphics[width=\linewidth]{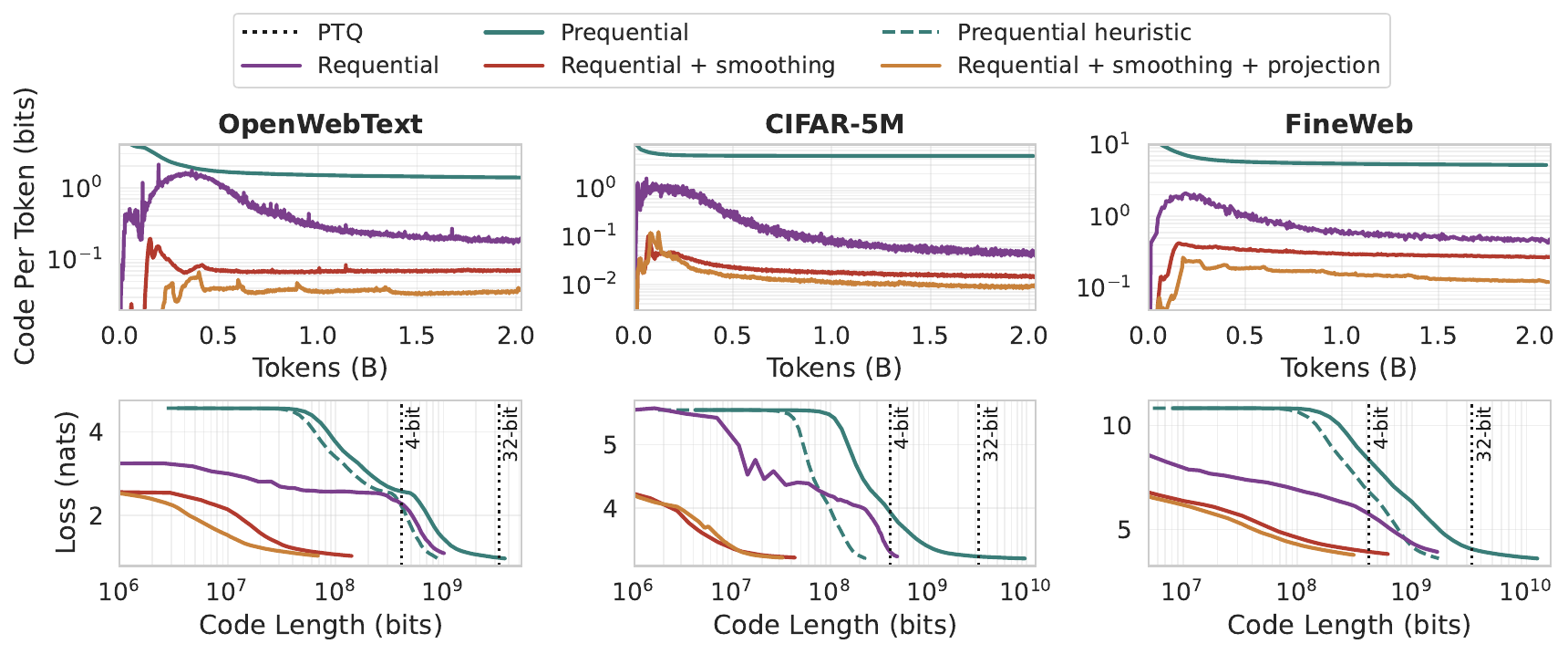}
\caption{\small
\textbf{Requential coding dominates the Pareto frontier of loss vs compressed model size.} Models with 100M parameters trained to 2B tokens on each dataset (columns), where $D$ counts the training tokens for the coded model.
\textbf{(Top)} Code length per token over training: the teacher-student KL for requential coding, with and without teacher smoothing and iso-loss projection, and the cross-entropy loss for prequential coding. \textbf{(Bottom)} Loss vs cumulative code length, with the idealized 4-bit PTQ and FP32 parameter sizes marked as vertical references. The prequential heuristic (dashed) subtracts the data code given the model from the prequential code, but does not form a valid model code.
}
\label{fig:train_curves}
\end{figure*}

\subsection{Benchmarking Compression of Transformers Trained on Text and Images}\label{sec:experiments}
We evaluate requential coding against other methods for compressing autoregressive transformers trained on OpenWebText with character-level tokenization, CIFAR-5M \citep{nakkiran2020deep} with one token per pixel, and FineWeb \citep{penedo2024fineweb} with the GPT-2 tokenizer. Each model has 100M parameters and is trained for 2B tokens.  We also include for reference the prequential heuristic (Section~\ref{sec:preq}), which subtracts the code length of the data given the final model and is \emph{not} a valid compressor. Weprovide full experiment details in Appendix~\ref{app:experiment}. Figure~\ref{fig:train_curves} shows that requential coding compresses the model far better than the alternatives. Its per-token cost (the teacher-student KL) runs one to two orders of magnitude below the prequential per-token cost (the teacher cross-entropy loss). With both teacher smoothing and iso-loss projection, the code shortens further without sacrificing performance. The resulting requential code dominates the Pareto frontier of loss vs code length and stops left of the 4-bit per parameter reference, with a significant gap on OpenWebText and CIFAR-5M, whereas the prequential code exceeds the FP32 parameter size and is a vacuous model compressor. Even the prequential heuristic, a non-rigorous estimate of the compressed model size, sits well above the requential code. 
\section{Understanding Learning and Generalization through Model Compression}
\label{sec:findings}
Beyond enabling models to be stored and transmitted with smaller file sizes, model compression provides a rigorous means to understand a variety of phenomena in machine learning. An extensive body of theory ties the shortest code that describes a model to how well it can generalize, such as the minimum description length principle \citep{rissanen1978modeling, grunwald2007minimum} and PAC-Bayes bounds \citep{shalev2014understanding,zhou2018non,lotfi2022pac,lotfi2023non}, and to how much useful information its training data contains, through epiplexity \citep{finzi2026entropy}. Operationalizing these theoretical insights, however, requires a strong model compressor: a weak one renders the theory vacuous or even misleading, for example making larger models appear more complex when they in fact generalize better and can be more compressible. Supplying this compressor is the major payoff of requential coding. In this section, we show that requential coding reveals a range of phenomena inaccessible to existing compressors, including that larger models can in fact be more compressible, generalization gap provably vanishes with scale for compute-optimal LLMs, and the distinction between the information content of a dataset and how much a model can learn from it. We provide full experiment details in Appendix~\ref{app:experiment}.

\begin{figure*}[t]
\centering
\includegraphics[width=\linewidth]{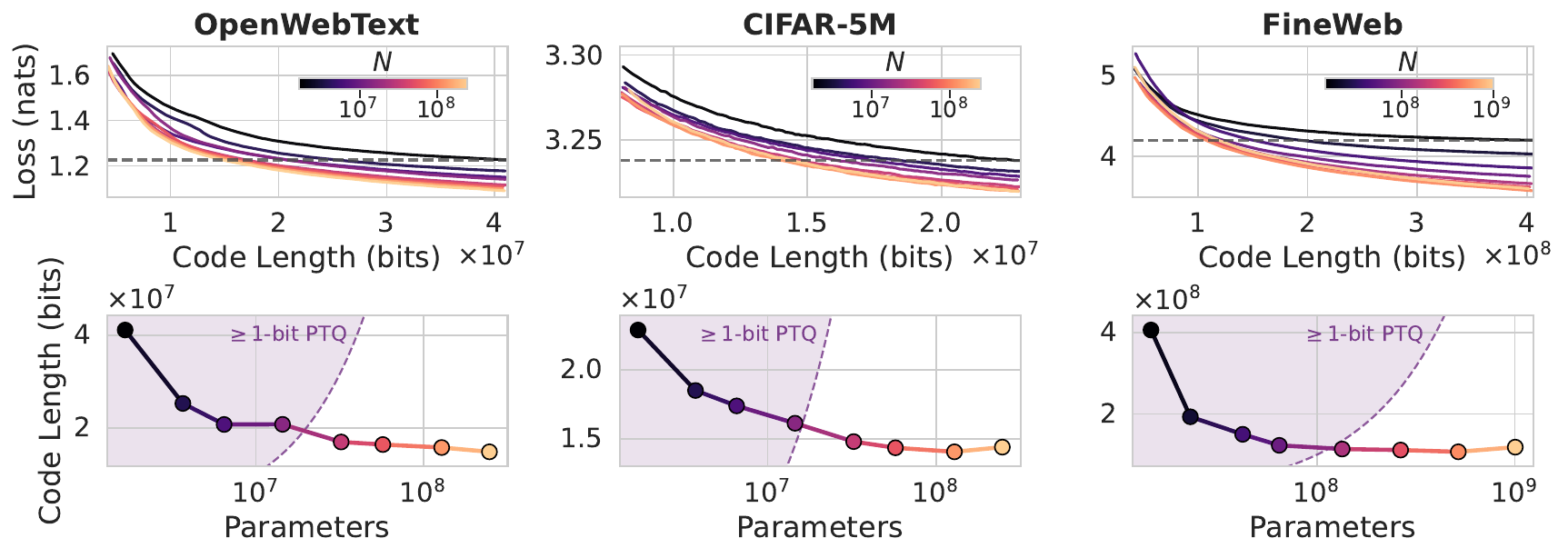}
\caption{\small\textbf{Requential coding reveals that larger models can be compressed to \emph{smaller} sizes despite having more parameters, holding the loss fixed.}
(\textbf{Top}) Student test loss vs code length curves traced by training: larger models reach a lower loss at a fixed code length. (\textbf{Bottom}) Requential code length to reach a fixed loss (final loss of the smallest model) \emph{shrinks} while parameter count grows by two orders of magnitude.
}
\label{fig:scaling-model}
\end{figure*}

\subsection{Larger Models and Ensembles Are More Compressible}
\label{sec:scale}
A long line of prior work has argued that overparameterized neural networks learn much simpler functions than their parameter counts suggest \citep{valleperez2019deep, frankle2018lottery, lotfi2022pac, lotfi2023non}, yet compressing the parameters from their final values alone is challenging, and increasingly intractable as model size $N$ grows and the information spreads over more parameters. The requential code is well suited to reveal this compressibility by encoding the actual information needed to produce the trained parameters. Figure~\ref{fig:scaling-model} (top) shows the effect of scaling model size $N$ on the test loss vs compressed size curves, showing larger models lie consistently below smaller ones. Equivalently, larger models can be compressed to \emph{fewer} bits at the same level of performance (bottom). The requential code directly translates the improved sample efficiency of larger models into fewer bits required to describe the model, a phenomenon parameter-based compression cannot leverage. As model size increases, their code length drops significantly below the 1-bit per parameter floor achievable by quantization.

\begin{wrapfigure}{r}{0.5\linewidth}
\vspace{-1.em}
\centering
\includegraphics[width=\linewidth]{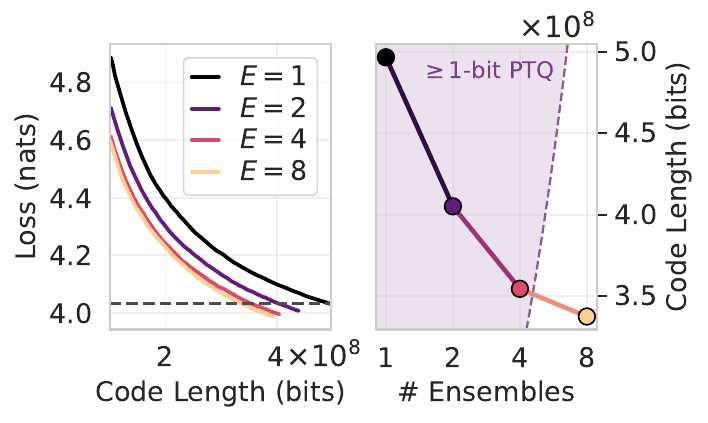}
\caption{\small\textbf{Larger ensembles are more compressible.}
(\textbf{Left}) Loss vs code length.
(\textbf{Right}) Code at fixed loss.
}
\label{fig:ensemble}
\vspace{-1.2em}
\end{wrapfigure}

The same holds for scaling an ensemble of models. Figure~\ref{fig:ensemble} trains an ensemble of $E \in \{1, 2, 4, 8\}$ members on FineWeb, each with 77M parameters trained for 1.5B tokens. All members share one teacher and one synthetic token stream, with the averaged student prediction used as the REC reference. A larger ensemble tracks the teacher more closely due to improved variance reduction, reaching a lower loss at a fixed code budget (left). The compressed size for reaching the final loss of a single model \emph{shrinks} as we ensemble over more models, despite the total parameter count increasing (right).

Since the compute used for compression and decompression in requential coding is coupled to the model size, the improved compression of larger models and ensembles owes in part to this extra compute. By spending more compute on compressing smaller models, for example using larger models to code their training data, it seems plausible the smaller models can be compressed further.

\subsection{Models Provably Generalize Better with Scale}
\label{sec:bounds}

\begin{figure*}[t]
\centering
\includegraphics[width=\linewidth]{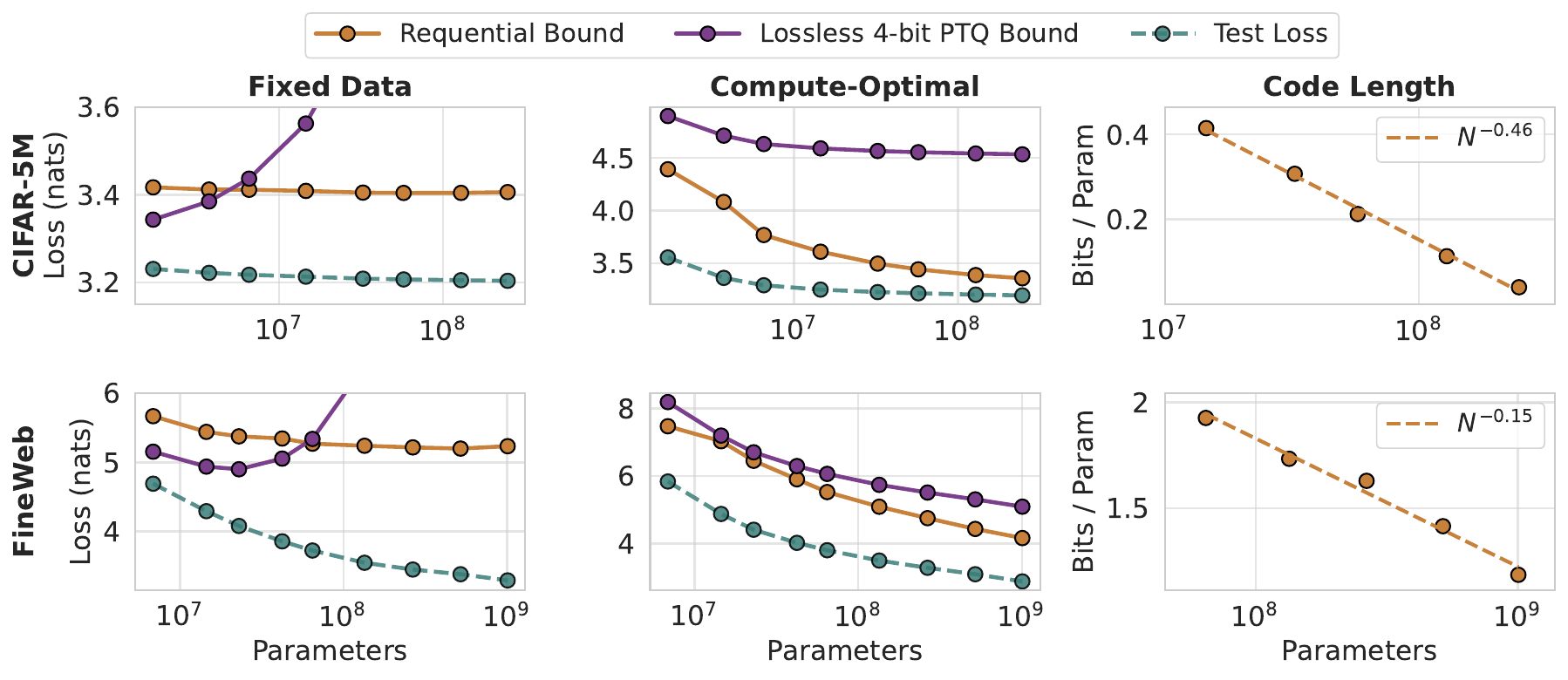}
\caption{\small\textbf{Requential codes certify that models generalize better with scale.} (\textbf{Left}) Holding data fixed, the requential code leads to PAC-Bayes generalization bounds that improve with model size in tandem with the true test loss, whereas the idealized PTQ bounds are competitive only for small, over-trained models. (\textbf{Middle}) With compute-optimal training ($D{=}20N$), our bound beats the $4$-bit PTQ bound representing the previous state of the art and tightens with scale. (\textbf{Right}) With compute-optimal training, compressed size per parameter decays as a power law with model size, predicting a generalization gap that vanishes with scale.}
\vspace{-0.5em}
\label{fig:generalization}
\end{figure*}

A certified compressed model size translates directly into a PAC-Bayes generalization guarantee. For a loss function taking values in $[0,\Delta]$, with probability at least $1-\delta,$ the expected risk $R(h)$ is bounded in terms of the empirical risk $\hat R(h)$ by:
\begin{equation}
    R(h) \le \hat{R}(h) + \Delta\sqrt{\frac{L(h)\ln 2 + \ln(1/\delta)}{2n}},
\end{equation}
where $L(h)$ is the length of a prefix-free code for $h$ and $n$ is the number of i.i.d. training examples \citep{shalev2014understanding,zhou2018non,lotfi2022pac,lotfi2023non}. For autoregressive language models, the prediction-smoothed per-token risk admits an analogous bound \citep{lotfi2024unlocking}, whose primary term \citet{finzi2025compute} reduce to a linear dependence on the per-token complexity: $R_s(h) \le \hat{R}(h) + C\ln V + (\Sigma+\sqrt2)\sqrt{C},$
where $C = \frac{1}{D}\qty[L(h)\ln 2 + \ln(|K|/\delta)]$ is the per-token complexity, $D$ is the number of training tokens, $V$ is vocabulary size, $\Sigma$ is a loss-variance term, $K$ is a finite set, and $\delta$ is the failure probability. We restate the full theorem as Theorem~\ref{thm:finzi_bound_rewritten} in Appendix~\ref{appx:discuss_bound}. Under this bound, a model that generalizes well must achieve low training loss while being compressible relative to the dataset size.  

Plugging in the requential code length $\widehat{L}_{\mathrm{req}}$ from Eq.~\eqref{eq:req_codelen} and the empirical risk (per-token cross-entropy) on the real training data (CIFAR-5M or FineWeb), we obtain a bound for the student model's expected risk on the real test data. $D$ counts the \emph{real} training tokens, namely the teacher's training tokens, which exceed the student's synthetic training tokens when iso-loss projection is used. We compare with an \emph{idealized} lossless post-training quantization (PTQ) oracle that quantizes weights from 32-bit to 4-bit precision without losing any performance. This approach is deliberately optimistic, as realistic PTQ methods lose accuracy due to quantization error \citep{lin2024awq, frantar2023gptq,tseng2024quip}. We evaluate both the test loss and the PTQ baselines on models trained normally on $D$ real tokens rather than the distilled student models, which sets a more demanding target for the requential bound.

\paragraph{Larger Models Generalize Better with the Same Amount of Data.} Figure~\ref{fig:generalization} (left) evaluates models of various sizes trained for one epoch over 2B tokens. The requential bound \emph{improves} with scale on both datasets mirroring the true test loss, showing that larger models generalize better by achieving lower training loss without disproportionately retaining more information from the training data, which would indicate memorization. Without requential coding, it would be infeasible to certify a nontrivial generalization bound for larger models, e.g., via PTQ.

\paragraph{Compute-Optimal LLMs Generalize Better with Scale.} Figure~\ref{fig:generalization} (middle) sets $D=20N$ \citep{hoffmann2022training}, the compute-optimal scaling regime where the number of tokens $D$ and parameters $N$ grow in fixed proportion. The requential bound improves with scale and outperforms the lossless 4-bit PTQ bound. Notably, \citet{finzi2025compute} set the previous state of the art for non-vacuous bounds on compute-optimal LLMs by quantizing weights with 4-bit GPTQ \citep{frantar2023gptq}, so beating the lossless idealization of this approach marks a substantial improvement. Remarkably, the bound tightens as the model grows. Figure~\ref{fig:generalization} (right) explains why: the code length per parameter $\widehat{L}_{\mathrm{req}}/N$, proportional to the per-token complexity $C$ ignoring lower-order terms, decays as a power law in $N$ on both datasets once the smallest models are excluded. Because the certified gap $C\ln V + (\Sigma + \sqrt{2})\sqrt{C}$ is controlled by $C$ ($\Sigma$ empirically converges to a constant \citep{finzi2025compute}), a decaying $C$ shrinks the gap. If the power law persists, then $C \to 0$ and the certified generalization gap \emph{vanishes} with scale. \citet{finzi2025compute} made a similar extrapolation based on the non-rigorous prequential heuristic, whereas requential coding certifies the same trend with a rigorous code.

It is worth taking a moment to appreciate the significance of this finding. The requential code reveals that compute-optimal scaling not only lowers training loss, but also produces increasingly compressible models, whose code length grows sublinearly in the parameter count. Had the latter failed, the modern scaling paradigm could have stopped yielding gains by running into an irreducible generalization gap where better training performance no longer translates to test performance. 
\clearpage

\subsection{Predicting Overfitting under Data Repetition Due to Memorization}\label{sec:overfitting}
\begin{wrapfigure}[11]{r}{0.34\linewidth}
\centering
\vspace{-1.8em}
\includegraphics[width=\linewidth]{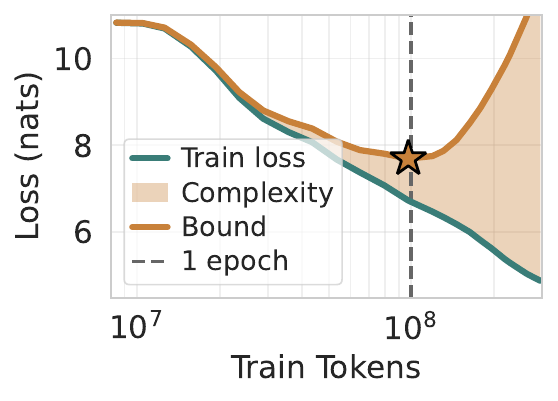}
\caption{\small\textbf{Predicting overfitting.} The bound predicts a gradual divergence between the train and test loss.}
\label{fig:u-shape}
\vspace{-1em}
\end{wrapfigure}
Applied to the multi-epoch regime, the same generalization bound now predicts the training and test loss gradually diverge as the model starts memorizing the training data. 
Figure~\ref{fig:u-shape} shows that the generalization bound produced by requential coding predicts a gradual buildup of overfitting as we train a model for multiple epochs on 100M tokens. The training loss decreases monotonically with more tokens seen, but the complexity penalty (terms involving $C$) rises even faster with data repetition, with the best bound attained around one epoch of training. The gradual information accumulation would have been invisible to parameter-based methods that access only the final parameters.

\subsection{How Much Information Can the Model Learn?}\label{sec:info}
\begin{wrapfigure}{r}{0.34\linewidth}
\centering
\vspace{-1.1em}
\includegraphics[width=\linewidth]{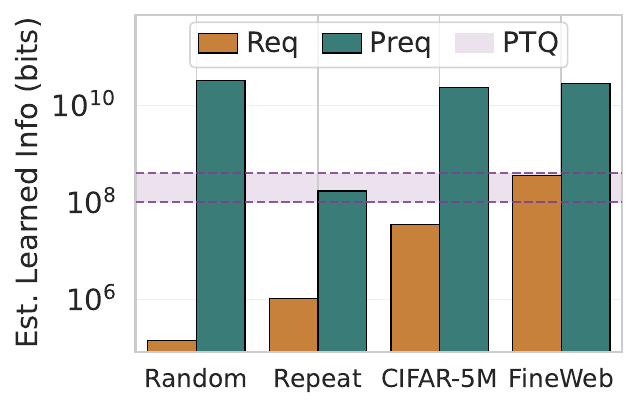}
\caption{\small\textbf{Compressed model size measures learnable information in data.} The requential code consistently separates structure learnable to the model from random information.}
\label{fig:info-barplot}
\vspace{-1.4em}
\end{wrapfigure}
Why is text pre-training uniquely useful for building general-purpose models despite similarly abundant information available in other domains? Answering this question requires separating the amount of information learnable to a model, i.e., epiplexity \citep{finzi2026entropy}, from the total information present in the dataset: uniformly random strings contain high random information (entropy), but a model trained on that dataset would learn very little. Decoupled from both parameter count and data entropy, the compressed model size under requential coding can meaningfully rank datasets by the amount of learnable structure (Figure~\ref{fig:info-barplot}): uniformly random strings and trivially repeating strings have little structure, images have substantial structure, and text has the most structure, in line with our expectations.  Each dataset has 5B tokens and the model is trained for one epoch to avoid memorizing random information.  By contrast, the prequential code is inflated by the data entropy, leading to similar and vacuous estimates for how much the model learns from random strings, images, and text, and quantization to between 1-bit (lower bound) and 4-bit per parameter (empirical bound on information capacity of LLMs \citep{morris2025much}) primarily reflects the model size and is insensitive to the data.

\section{Discussion}
\label{sec:discuss}

We have introduced requential coding, a novel method for compressing generative models. By coding a training process built from self-generated data, the requential code depends on neither the parameter count nor the data entropy, compressing neural networks to far smaller sizes than previously shown, with an advantage that grows with model scale. As a result, its compressed model size proves to be a powerful tool for understanding learning phenomena beyond the reach of prior compression methods.

While a significant step toward revealing how compressible neural networks really are, there are
several open questions that could form the basis for exciting future work. Conceptually, requential coding provides a lossless code for the student, not the teacher. Practically, the explicit dependence on a model's training process makes requential coding a strong compressor, but also necessitates running the teacher-student training to evaluate the code length, unlike parameter-based codes such as quantization that can be applied directly to the trained parameters. Additionally, at this stage, requential coding is primarily a tool to evaluate the compressed model size rather than to transmit the model, as the REC encoder requires runtimes that scale exponentially in the KL divergence.

Moreover, requential coding compresses a model's training process, but does not account for the fact that some of the information is not retained over the course of training. Information learned early may be gradually lost as the model trains on new data, and could in principle be removed from the code. It is a particularly exciting open question to understand whether one could leverage forgotten information to reduce the code length, as the code length presently only grows with training steps and never decreases. We believe addressing this gap holds significant potential for further pushing the limits of model compression and explaining generalization phenomena. Finally, it would be interesting to explore using our bounds to prescribe training practices that improve generalization.

\paragraph{Acknowledgements.}
We thank Yiding Jiang, Pavel Izmailov, Ethan Baron and Eric Elmoznino for helpful discussions. This work was supported by NSF CAREER IIS-2145492, NSF CDS\&E-MSS 2134216, DARPA AIQ
HR00112590066 and Google's TPU Research Cloud (TRC) program. SQ is supported by the Two Sigma PhD Fellowship. Calculations in Appendix~\ref{appx:realized-codelength} were assisted by GPT 5.5.

\bibliography{ref}
\bibliographystyle{plainnat}

\newpage
\appendix

\section*{Appendix Outline}
The appendix is organized as follows. Appendix~\ref{appx:compute_overhead} proves the expected code length bound of Eq.~\eqref{eq:req_codelen} and analyzes the runtimes of code length evaluation, encoding, and decoding, including the tradeoff between code length and encoding time through the REC block size (Figure~\ref{fig:block-size-tradeoff}). Appendix~\ref{appx:realized-codelength} shows that the realized code length concentrates around its cumulative conditional mean, so the reported bound $\widehat L_\mathrm{req}$ certifies the actual transmitted code length with high probability. Appendix~\ref{app:experiment} provides full experiment details, including how code lengths are computed, the datasets, the combined encoder with teacher smoothing and iso-loss projection (Algorithm~\ref{alg:smoothing-projection}), and per-figure configurations. Appendix~\ref{appx:discuss_bound} restates the generalization bound of \citet{finzi2025compute} evaluated in Section~\ref{sec:bounds}.
\section{Code Length and Runtimes}\label{appx:proof}\label{appx:compute_overhead}
\subsection{Expected Code Length}
\paragraph{Universal Integer Code for the Index.}
At step $t$, the current student $P_t$ defines the public proposal sequence, the encoder uses the teacher $Q_t$ to select a proposal, and the decoder recovers the selected sample $X_t$ from the prefix-free message $m_t$. Let $\mathcal F_{t-1}$ be the history before this REC call, and let $J_t$ denote the selected proposal index. The original PFR and ORC analyses obtain their sharpest entropy bound by comparing the selected index with a Zipf distribution whose exponent is tuned based on $\mathrm{KL}(Q_t\|P_t)$ \citep{li2018strong,theis2022algorithms}. In our setup, the decoder does not have access to the teacher $Q_t$ and therefore cannot generally compute the tuned exponent. We instead encode $J_t$ with the parameter-free Elias delta universal integer code, which adds negligible overhead while avoiding any need for the decoder to know or estimate the teacher-student KL. For every integer $j\geq1$, its code length is
\begin{equation}
    \ell_\Delta(j)
    =
    \left\lfloor\log j\right\rfloor
    +
    2\left\lfloor
      \log\!\left(\left\lfloor\log j\right\rfloor+1\right)
    \right\rfloor
    +1,
    \label{eq:elias-delta-length}
\end{equation}
and therefore
\begin{equation}
    \ell_\Delta(j)
    \leq
    \log j
    +2\log(1+\log j)
    +1.
    \label{eq:elias-pointwise-upper}
\end{equation}
Define
\begin{align}
    \beta
    &=
    e^{-1}\log e+1,
    \\
    \kappa
    &=
    1+\beta+2\log(1+\beta)
    <5.21.
    \label{eq:universal-code-constants}
\end{align}

The log-index bound for the PFR \citep{li2018strong}, which carries over to ORC \citep{theis2022algorithms}, gives
\begin{equation}
    \mathbb E[\log J_t\mid\mathcal F_{t-1}]
    \leq
    \mathrm{KL}(Q_t\|P_t)+\beta.
    \label{eq:pfr-expected-log-index}
\end{equation}

\begin{proposition}[Conditional mean REC message length]
\label{prop:conditional-mean-rec-message}
If the selected index $J_t$ is encoded with the Elias delta code, so that $\ell_t=\ell_\Delta(J_t)$, then
\begin{equation}\label{eq:orc_block}
\mathbb{E}[\ell_t\mid\mathcal F_{t-1}]
\le
\mathrm{KL}(Q_t\|P_t)
+2\log\!\left(1+\mathrm{KL}(Q_t\|P_t)\right)
+\kappa
\end{equation}
for every REC call $t$.
\end{proposition}

\begin{proof}
Let
\begin{equation}
    A_t
    =
    \mathbb E[\log J_t\mid\mathcal F_{t-1}].
\end{equation}
By the pointwise Elias delta bound in Eq.~\eqref{eq:elias-pointwise-upper} and Jensen's inequality,
\begin{align}
    \mathbb E[\ell_t\mid\mathcal F_{t-1}]
    &\leq
    A_t+
    2\log(1+A_t)
    +1.
\end{align}
Since $u+2\log(1+u)+1$ is increasing for $u\geq0$, combining this with Eq.~\eqref{eq:pfr-expected-log-index} gives
\begin{align}
    \mathbb E[\ell_t\mid\mathcal F_{t-1}]
    &\leq
    \mathrm{KL}(Q_t\|P_t)
    +\beta
    +2\log\!\left(1+\mathrm{KL}(Q_t\|P_t)+\beta\right)
    +1
    \\
    &\leq
    \mathrm{KL}(Q_t\|P_t)
    +2\log\!\left(1+\mathrm{KL}(Q_t\|P_t)\right)
    +1+\beta+2\log(1+\beta),
\end{align}
where the final step uses
\begin{equation}
    1+\mathrm{KL}(Q_t\|P_t)+\beta
    \leq
    \left(1+\mathrm{KL}(Q_t\|P_t)\right)(1+\beta).
\end{equation}
The definition of $\kappa$ proves Eq.~\eqref{eq:orc_block}.
\end{proof}

More efficient choices such as the log-star universal integer codes \citep{rissanen1983universal} can further reduce the logarithmic overhead, but this term is already negligible in our large-batch applications compared with the leading KL term.

\begin{theorem}[Conditional mean code length]
\label{thm:conditional-mean-codelength}
Define the cumulative one-step conditional mean code length
\begin{equation}
\overline{L}_{\mathrm{req}}
:=
\sum_{t=0}^{T-1}\mathbb{E}[\ell_t\mid\mathcal F_{t-1}].
\end{equation}
If each selected index is encoded as in Proposition~\ref{prop:conditional-mean-rec-message}, then
\begin{equation}
\overline{L}_{\mathrm{req}}
\le
\sum_{t=0}^{T-1}
\left[
\mathrm{KL}(Q_t\|P_t)
+2\log\!\left(1+\mathrm{KL}(Q_t\|P_t)\right)
+\kappa
\right].
\end{equation}
\end{theorem}

\begin{proof}
Proposition~\ref{prop:conditional-mean-rec-message} gives the claimed upper bound on $\mathbb{E}[\ell_t\mid\mathcal F_{t-1}]$ for every REC call $t$. Summing these inequalities over $t=0,\ldots,T-1$ proves the result.
\end{proof}
Appendix~\ref{appx:realized-codelength} relates this cumulative conditional mean to the realized code length.

\subsection{Runtime}
We normalize runtime by one ordinary training run on the same $D$ tokens and count the FLOPs for a backward pass as two forward passes. We use $B$ to denote the batch size measured in tokens.
\paragraph{Code length evaluation.}
Code length evaluation does not require actually transmitting the model. As in the main text, we replace REC encoding and decoding by the equivalent stochastic process that samples each $X_t$ directly from the teacher $Q_t$. If the teacher shares the student's architecture and advances by training on real data with the same batch size, each step consists of one teacher forward pass to sample $X_t$, one student forward-backward pass on $X_t$, and one teacher forward-backward pass on real data. Thus
\begin{equation}
\frac{T_{\mathrm{eval}}}{T_{\mathrm{train}}}
\approx
\frac{3D+D+3D}{3D}
=
\frac{7}{3}.
\end{equation}
If the teacher checkpoints are already available, the only extra compute beyond ordinary student training is the teacher forward pass used to sample $X_t$, giving a $0.33\times$ overhead.

\paragraph{Encoding.} Actually transmitting a model is more expensive because the encoder needs to draw many proposals from $P_t$. Rejection sampling draws a number of proposals that scales with the worst-case likelihood ratio and achieves suboptimal code length. An exact REC construction achieving the $\mathrm{KL}(Q_t\|P_t) + \log(1+\mathrm{KL}(Q_t\|P_t)) + O(1)$ expected code length, like PFR \citep{li2018strong}, uses an unbounded number of proposals. Approximate methods like ORC draw about $2^{\mathrm{KL}(Q_t\|P_t) + \rho}$ proposals at step $t$, where the parameter $\rho>0$ controls the approximation of the ORC sampling distribution to $Q_t$. We illustrate the runtime for ORC. For each proposal, the student sampling forward pass produces the proposal and its log probability under $P_t$, while the encoder also needs a teacher forward pass to evaluate its log probability under $Q_t$. Treating one proposal as two forward passes over $B$ tokens gives
\begin{equation}
\frac{T_{\mathrm{enc}}}{T_{\mathrm{train}}}
\approx
2+
\frac{2B}{3D}
\sum_{t=0}^{T-1}
2^{\mathrm{KL}(Q_t\|P_t)+\rho}.
\end{equation}
The leading $2$ accounts for the teacher and student training work done during encoding; the summation is the ORC search cost. 

\paragraph{Decoding.} The decoding time, in contrast, is unaffected by the block size. The decoder never evaluates the teacher; it regenerates each selected $X_t$ from the current student and shared randomness, then updates the student. For an autoregressive model, the sampling forward pass can be reused for the backward update, giving
\begin{equation}
\frac{T_{\mathrm{dec}}}{T_{\mathrm{train}}}\approx 1.
\end{equation}
If the sampling pass is not reused, the decoding cost is closer to $4/3$ ordinary training.

\paragraph{Block Size Tradeoff.}
When encoding time matters, we can accept a longer code in exchange for a shorter encoding time by dividing each batch into smaller blocks and transmitting one block at a time. We plot this trade-off in Figure~\ref{fig:block-size-tradeoff}.
Changing the REC block size does not change the algorithm; it only changes what one step's sample $X_t$ denotes. Suppose the logged run uses batch size $B$, and define the per-sample KL
\begin{equation}
\epsilon_t
:=
\frac{1}{B}\mathrm{KL}(Q_t\|P_t).
\end{equation}
For a candidate REC block size $1\leq B'\leq B$, we estimate the code length by splitting each original batch into blocks of size $B'$ and treating each such block as having KL $B'\epsilon_t$. Ignoring rounding when $B'$ does not divide $B$, this gives
\begin{equation}
\widehat{L}_{\mathrm{req}}(B')
=
\sum_{t=0}^{T-1}
\left[
B\epsilon_t
+
\frac{B}{B'}
\left(
2\log(1+B'\epsilon_t)
+\kappa
\right)
\right].
\end{equation}
Thus the leading cumulative KL is fixed, while the logarithmic and constant-$\kappa$ universal-code overheads increase with the number of REC messages, and hence decrease with $B'$. To make Figure~\ref{fig:block-size-tradeoff}, we reuse one 104M-parameter FineWeb run with teacher smoothing and no iso-loss projection with $B=524288$ tokens, and apply this rescaling for each plotted $B'$. The proposal count is set to $2^{B'\epsilon_t}$ with the finite-candidate slack $\rho$ set to 0, sufficient to illustrate the exponential scaling.

\begin{figure}[H]
\centering
\begin{minipage}[c]{0.42\linewidth}
\centering
\includegraphics[width=\linewidth]{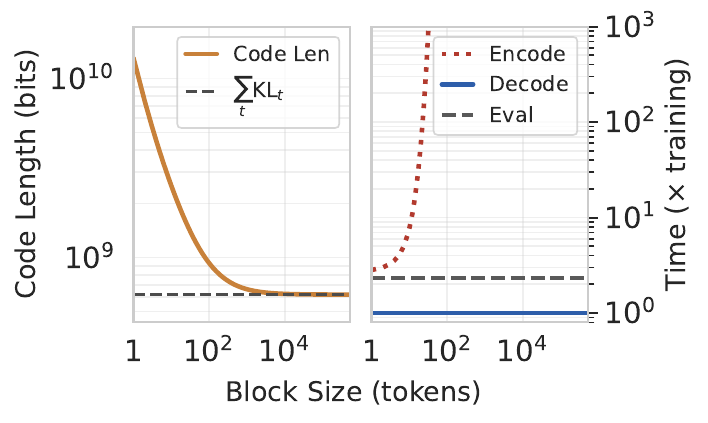}
\end{minipage}\hfill
\begin{minipage}[c]{0.54\linewidth}
\caption{\small\textbf{Trading code length against encoding time.} Shown for a 104M-parameter model trained to the Chinchilla budget on FineWeb. (\textbf{Left}) The code length drops to its floor as the per-block overhead is amortized. (\textbf{Right}) Actual decoding is close to ordinary training, evaluation without REC search costs $2.33\times$ ordinary training, and encoding time grows exponentially with block size. Time is normalized to ordinary training on the same tokens.}
\label{fig:block-size-tradeoff}
\end{minipage}
\end{figure}

\section{Bounding the Realized Requential Codelength}
\label{appx:realized-codelength}

Appendix~\ref{appx:compute_overhead} proves the one-step conditional mean bound used for expected codelengths. A compression certificate, however, depends on the realized prefix-free code length. In addition, the samples selected by REC determine subsequent training updates, so the teacher-student trajectory and all later coding distributions are themselves random. This section separates these effects. We first compare the realized codelength with the sum of its one-step conditional means. We then bound the conditional variance of one message for the Poisson functional representation (PFR) code of \citet{li2018strong}, the infinite-candidate limit of ordered random coding (ORC) \citep{theis2022algorithms}. Thus, the bound proved in this section certifies the variance only in the limit as we draw an arbitrarily large number of proposals in ORC. This is impractical for actual encoding, but remains theoretically sound for our purpose: we establish that there exists a code with the claimed mean and variance, which is what is relevant for certifying compressibility and generalization bounds, even though \emph{finding} that code (encoding) can be computationally impractical.

We index REC calls by $t=0,\ldots,T-1$. Let $\mathcal F_{t-1}$ contain the complete requential history immediately before call $t$, including the current teacher and student distributions. For $t=0$, $\mathcal F_{-1}$ denotes the initial shared state.

\subsection{A martingale decomposition of the realized codelength}

Let $m_t$ be the prefix-free message sent at call $t$, let
\begin{equation}
    \ell_t
    =
    |m_t|
\end{equation}
be its realized length, and define
\begin{equation}
    \overline\ell_t
    =
    \mathbb E[\ell_t\mid\mathcal F_{t-1}].
\end{equation}
The realized codelength and its cumulative conditional mean are
\begin{equation}
    L
    =
    \sum_{t=0}^{T-1}\ell_t,
    \qquad
    \overline L
    =
    \sum_{t=0}^{T-1}\overline\ell_t.
\end{equation}
The quantity $\overline L$ is random because it is evaluated along the realized teacher-student trajectory.

\begin{proposition}[Cumulative coding fluctuation]
\label{prop:cumulative-coding-fluctuation}
The realized codelength satisfies
\begin{align}
    \mathbb E[L-\overline L]
    &=
    0,
    \\
    \mathbb E\left[
      (L-\overline L)^2
    \right]
    &=
    \sum_{t=0}^{T-1}
    \mathbb E\left[
      \operatorname{Var}(\ell_t\mid\mathcal F_{t-1})
    \right].
    \label{eq:martingale-variance-identity}
\end{align}
\end{proposition}

\begin{proof}
The differences
\begin{equation}
    Z_t
    =
    \ell_t-\overline\ell_t
\end{equation}
form a martingale-difference sequence. Hence their cross terms have zero expectation, and
\begin{equation}
    \mathbb E\left[
      \left(\sum_{t=0}^{T-1}Z_t\right)^2
    \right]
    =
    \sum_{t=0}^{T-1}\mathbb E[Z_t^2]
    =
    \sum_{t=0}^{T-1}
    \mathbb E\left[
      \operatorname{Var}(\ell_t\mid\mathcal F_{t-1})
    \right].
\end{equation}
\end{proof}

\paragraph{Interpretation.}
Equation~\eqref{eq:martingale-variance-identity} averages over all randomness, including both the random REC messages and the random teacher-student trajectory they induce. It says that the expected squared difference between the actual code length and the cumulative one-step conditional mean is controlled by the accumulated conditional variance. Thus, whenever the accumulated variance grows more slowly (typically $O(T)$) than the square of the accumulated mean (typically $O(T^2)$), the relative coding fluctuation is small.

This result does not show that $\overline L$ is close to the scalar $\mathbb E[L]$ on any particular run, which is not what we are interested in showing. The random REC samples alter subsequent training updates and therefore alter future conditional means. Controlling variation of $\overline L$ across runs would require a separate stability or concentration result for the training trajectory that is not relevant to our investigation.

\subsection{Variance of one PFR message}

We now bound the conditional variance appearing in Eq.~\eqref{eq:martingale-variance-identity}. Fix discrete distributions $P$ and $Q$ on a countable alphabet $\mathcal X$, and assume $Q(x)>0$ only if $P(x)>0$. Define the information density
\begin{equation}
    I(x)
    =
    \log\frac{Q(x)}{P(x)},
    \qquad
    I_+(x)
    =
    \max\{I(x),0\}.
\end{equation}

The PFR construction draws independent proposals $\widetilde X_1,\widetilde X_2,\ldots\sim P$ and lets $0<T_1<T_2<\cdots$ be the arrival times of a unit-rate Poisson process. It selects
\begin{equation}
    J
    =
    \arg\min_{i\geq1}
    T_i\frac{P(\widetilde X_i)}{Q(\widetilde X_i)},
    \qquad
    X
    =
    \widetilde X_J.
    \label{eq:pfr-construction}
\end{equation}
Then $X\sim Q$ \citep{li2018strong}.

\begin{lemma}[Variance of the PFR log-index]
\label{lem:pfr-log-index-variance}
Let $J$ and $X$ be generated by Eq.~\eqref{eq:pfr-construction}, and define
\begin{equation}
    \gamma
    =
    1+\frac{2}{\ln 2}
    <4.
\end{equation}
Then
\begin{align}
    \mathbb E\left[
      \left(\log J-I_+(X)\right)^2
    \right]
    &\leq
    \gamma^2,
    \label{eq:log-index-residual}
    \\
    \sqrt{\operatorname{Var}(\log J)}
    &\leq
    \sqrt{\operatorname{Var}_{X\sim Q}(I(X))}
    +\gamma.
    \label{eq:log-index-variance}
\end{align}
The same bounds hold conditionally when $P$ and $Q$ are determined by the preceding history.
\end{lemma}

\begin{proof}
Write
\begin{equation}
    R(x)
    =
    \frac{Q(x)}{P(x)},
\end{equation}
and define the analytical quantity
\begin{equation}
    a(x)
    =
    \sum_{y\in\mathcal X}
    P(y)\left(R(x)-R(y)\right)_+.
    \label{eq:analytical-a}
\end{equation}
The encoder does not evaluate $a(x)$; it is used only to characterize the PFR index. The conditional Poisson calculation in the proof of the strong functional representation lemma gives
\begin{equation}
    J-1
    \mid
    X=x,\Lambda=\lambda
    \sim
    \operatorname{Poisson}(\lambda\, a(x)),
    \qquad
    \Lambda\sim\operatorname{Exp}(1),
\end{equation}
where $\Lambda$ is independent of $X$. Integrating over $\Lambda$ gives
\begin{equation}
    \Pr(J=j\mid X=x)
    =
    \frac{a(x)^{j-1}}{(1+a(x))^j},
    \qquad
    j=1,2,\ldots.
    \label{eq:conditional-geometric}
\end{equation}
Thus $J\mid X=x$ is geometric on $\{1,2,\ldots\}$ with conditional mean
\begin{equation}
    M(x)
    =
    1+a(x).
\end{equation}

For a geometric random variable $J$ with mean $M\geq1$, elementary tail bounds give
\begin{equation}
    \mathbb E\left[
      \left(\log J-\log M\right)^2
    \right]
    \leq
    \frac{4}{(\ln 2)^2}.
    \label{eq:geometric-log-residual}
\end{equation}
Indeed, for $Z=\ln(J/M)$ and every $u\geq0$, one has $\Pr(Z\leq-u)\leq e^{-u}$ and $\Pr(Z\geq u)\leq e^{1-e^u}\leq e^{-u}$; integrating these tails yields Eq.~\eqref{eq:geometric-log-residual}.

Since $\sum_yP(y)R(y)=1$, Jensen's inequality and the elementary upper bound $(R(x)-R(y))_+\leq R(x)$ give
\begin{equation}
    (R(x)-1)_+
    \leq
    a(x)
    \leq
    R(x).
\end{equation}
Consequently,
\begin{equation}
    I_+(x)
    \leq
    \log M(x)
    \leq
    I_+(x)+1.
    \label{eq:mean-index-information-comparison}
\end{equation}
Combining Eqs.~\eqref{eq:geometric-log-residual} and \eqref{eq:mean-index-information-comparison} with Minkowski's inequality proves Eq.~\eqref{eq:log-index-residual}.

Finally, write $W=\log J-I_+(X)$. Equation~\eqref{eq:log-index-residual} gives $\mathbb E[W^2]\leq\gamma^2$, and therefore
\begin{align}
    \sqrt{\operatorname{Var}(\log J)}
    &\leq
    \sqrt{\operatorname{Var}(I_+(X))}
    +
    \sqrt{\operatorname{Var}(W)}
    \\
    &\leq
    \sqrt{\operatorname{Var}(I(X))}
    +\gamma.
\end{align}
The final inequality follows because $x\mapsto x_+$ is $1$-Lipschitz and hence cannot increase the variance of a scalar random variable.
\end{proof}

\subsection{Variance of one universally coded PFR message}

We use the same Elias delta code $\ell_\Delta$ from Eq.~\eqref{eq:elias-delta-length}.
\begin{lemma}[Variance of one universally coded PFR message]
\label{lem:pfr-message-variance}
Let
\begin{equation}
    \ell
    =
    \ell_\Delta(J).
\end{equation}
Then
\begin{equation}
    \sqrt{\operatorname{Var}(\ell)}
    \leq
    \gamma\left(
      \sqrt{\operatorname{Var}_{X\sim Q}(I(X))}
      +\gamma
    \right)
    +\frac{3}{2}.
    \label{eq:single-message-variance}
\end{equation}
The same bound holds conditionally on a preceding history when $P$ and $Q$ are fixed given that history.
\end{lemma}

\begin{proof}
For $u\geq0$, define
\begin{equation}
    g(u)
    =
    u+2\log(1+u)+1.
\end{equation}
Let
\begin{equation}
    U
    =
    \log J.
\end{equation}
A direct comparison with Eq.~\eqref{eq:elias-delta-length} gives
\begin{equation}
    0
    \leq
    g(U)-\ell_\Delta(J)
    <3.
    \label{eq:elias-rounding-residual}
\end{equation}
To see this, let $n=\lfloor U\rfloor$ and $q=\lfloor\log(n+1)\rfloor$. Then $U-n<1$ and $\log(1+U)-q<1$, while both differences are nonnegative. Equation~\eqref{eq:elias-rounding-residual} follows from
\begin{equation}
    g(U)-\ell_\Delta(J)
    =
    (U-n)
    +2\left(\log(1+U)-q\right).
\end{equation}
The function $g$ is $\gamma$-Lipschitz on $[0,\infty)$ because
\begin{equation}
    g'(u)
    =
    1+\frac{2}{(1+u)\ln 2}
    \leq
    \gamma.
\end{equation}
A $\gamma$-Lipschitz function increases the standard deviation by at most a factor of $\gamma$, and a random variable supported on an interval of width three has variance at most $9/4$. Hence, by Minkowski's inequality,
\begin{align}
    \sqrt{\operatorname{Var}(\ell)}
    &\leq
    \sqrt{\operatorname{Var}(g(U))}
    +\frac{3}{2}
    \\
    &\leq
    \gamma\sqrt{\operatorname{Var}(U)}
    +\frac{3}{2}.
\end{align}
Applying Eq.~\eqref{eq:log-index-variance} proves Eq.~\eqref{eq:single-message-variance}.
\end{proof}

\subsection{Summing over a full requential coding trajectory}

At call $t$, let $P_t$ and $Q_t$ be the student and teacher distributions over the entire object $X_t$ transmitted by that REC call. For example, in the language-model experiments, $X_t$ may be a complete batch of sequences. Define
\begin{align}
    I_t(x)
    &=
    \log\frac{Q_t(x)}{P_t(x)},
    \\
    \mu_t
    &=
    \mathbb E_{X\sim Q_t}[I_t(X)]
    =
    \mathrm{KL}(Q_t\|P_t),
    \\
    s_t^2
    &=
    \operatorname{Var}_{X\sim Q_t}(I_t(X)).
    \label{eq:call-level-mu-s}
\end{align}
These quantities are determined by the history before call $t$. Let $J_t$ be the PFR index selected at call $t$ and encode it using the Elias delta code. Define
\begin{align}
    h_t
    &=
    2\log(1+\mu_t)
    +\kappa,
    \\
    r_t
    &=
    \gamma(s_t+\gamma)
    +\frac{3}{2}.
    \label{eq:trajectory-terms}
\end{align}
Applying Proposition~\ref{prop:conditional-mean-rec-message} and Lemma~\ref{lem:pfr-message-variance} conditionally at each call gives
\begin{align}
    \overline\ell_t
    &\leq
    \mu_t+h_t,
    \label{eq:conditional-message-mean-upper}
    \\
    \operatorname{Var}(\ell_t\mid\mathcal F_{t-1})
    &\leq
    r_t^2.
    \label{eq:conditional-message-variance}
\end{align}
Define
\begin{equation}
    M
    =
    \sum_{t=0}^{T-1}\mu_t,
    \qquad
    H
    =
    \sum_{t=0}^{T-1}h_t,
    \qquad
    V
    =
    \sum_{t=0}^{T-1}r_t^2.
    \label{eq:cumulative-terms}
\end{equation}
Here $M$ is the cumulative teacher-student KL evaluated along the realized trajectory, the quantity used as the leading codelength term in the paper.

\begin{theorem}[Upper bound on the realized codelength]
\label{thm:realized-codelength-upper}
The cumulative conditional mean satisfies
\begin{equation}
    \overline L
    \leq
    M+H.
    \label{eq:conditional-cumulative-mean-upper}
\end{equation}
The fluctuation satisfies
\begin{equation}
    \mathbb E\left[
      (L-\overline L)^2
    \right]
    \leq
    \mathbb E[V].
    \label{eq:cumulative-fluctuation-upper}
\end{equation}
Moreover, for every $\delta\in(0,1)$, with probability at least $1-\delta$,
\begin{equation}
    L
    \leq
    M+H
    +
    \sqrt{\frac{\mathbb E[V]}{\delta}}.
    \label{eq:realized-codelength-upper}
\end{equation}
\end{theorem}

\begin{proof}
Summing Eq.~\eqref{eq:conditional-message-mean-upper} gives Eq.~\eqref{eq:conditional-cumulative-mean-upper}. Substituting Eq.~\eqref{eq:conditional-message-variance} into Eq.~\eqref{eq:martingale-variance-identity} gives Eq.~\eqref{eq:cumulative-fluctuation-upper}. Finally, Chebyshev's inequality gives
\begin{equation}
    \Pr\left(
      L-\overline L
      \geq
      \sqrt{\frac{\mathbb E[V]}{\delta}}
    \right)
    \leq
    \delta.
\end{equation}
Combining this event with Eq.~\eqref{eq:conditional-cumulative-mean-upper} proves Eq.~\eqref{eq:realized-codelength-upper}.
\end{proof}

The theorem gives the appropriate interpretation of the cumulative KL used in the paper: the reported code length $\widehat{L}_\mathrm{req} = M+H$ from Eq.~\eqref{eq:req_codelen} is an upper bound on the cumulative one-step conditional mean $\overline{L}$ along the realized trajectory, which differs from the actual code length by a fluctuation controlled by Eq.~\eqref{eq:cumulative-fluctuation-upper}. We now show that in practice this fluctuation is tiny with high probability and thus can be ignored.

\paragraph{Getting a sense of scale.}
Suppose that every call satisfies $\mu_t\geq\mu$ and $s_t\leq s$. Let
\begin{equation}
    r
    =
    \gamma(s+\gamma)+\frac{3}{2}.
\end{equation}
Then $\mathbb E[V]\leq T r^2$, and Eq.~\eqref{eq:realized-codelength-upper} with $\delta=0.01$ gives the informal 99\% estimate
\begin{equation}
    L
    \leq
    \underbrace{M}_{\text{KL}}
    +
    \underbrace{H}_{\text{universal-code overhead}}
    +
    \underbrace{10r\sqrt T}_{\text{fluctuation}}.
    \label{eq:informal-99-upper}
\end{equation}
Since $M\geq T\mu$, the fluctuation term is at most a fraction
\begin{equation}
    \frac{10r}{\mu\sqrt T}
    \label{eq:informal-relative-surcharge}
\end{equation}
of the cumulative KL.

To translate this into sample units, suppose each call encodes $B$ conditionally independent samples (for sequence data, $B$ should be measured in sequences since tokens within a sequence are generally not independent). If the information density of one sample has mean $\mu'$ and standard deviation $s'$, then
\begin{equation}
    \mu
    =
    B\mu',
    \qquad
    s
    =
    \sqrt B\,s'.
\end{equation}
For large $B$, the additive constants in $r$ are negligible, and the conservative universal-code variance bound gives
\begin{equation}
    \frac{10r}{\mu\sqrt T}
    \approx
    10\gamma
    \frac{s'/\mu'}{\sqrt{BT}}.
    \label{eq:informal-sample-level-99}
\end{equation}
Accordingly, a sufficient leading-order condition for the fluctuation to be below $1\%$ of cumulative KL with at least $99\%$ probability is
\begin{equation}
    BT
    \geq
    10^6\gamma^2
    \qty(\frac{s'}{\mu'})^2.
    \label{eq:informal-one-percent-condition}
\end{equation}
The factor $\gamma<4$ comes from a uniform Lipschitz bound for the Elias delta length as a function of the log-index and is conservative. The important scaling is that the relative fluctuation decays as $(BT)^{-1/2}$.

Figure~\ref{fig:realized-codelength-dev} evaluates these quantities on a $100$M-parameter FineWeb run. The resulting estimate falls below $1\%$ within the training budget, with and without iso-loss projection.

\begin{figure}[H]
\centering
\includegraphics[width=0.8\linewidth]{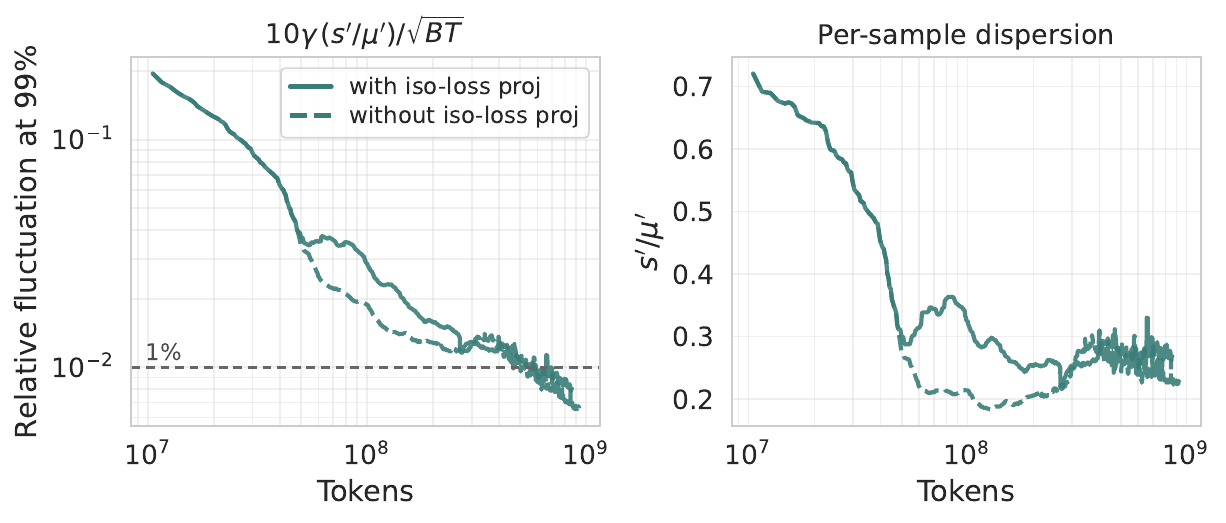}
\caption{\small\textbf{The realized requential codelength concentrates tightly around its cumulative conditional mean.} For a $100$M-parameter FineWeb model, with (solid) and without (dashed) iso-loss projection, we plot the sense-of-scale quantities versus tokens. \textbf{(Left)} the conservative estimated relative fluctuation at $99\%$ probability, $10\gamma\,(s'/\mu')/\sqrt{BT}$ from Eq.~\eqref{eq:informal-sample-level-99}, where $\mu'$ and $s'$ are the per-sequence mean and standard deviation of the information density and $BT$ is the number of coded sequences. \textbf{(Right)} the per-sample dispersion $s'/\mu'$. The estimated fluctuation falls below $1\%$ (dashed line) within the training budget.}
\label{fig:realized-codelength-dev}
\end{figure}

\section{Experiment Details}
\label{app:experiment}

Unless otherwise stated, we use the GPT-2 \citep{radford2019language} transformer architecture with 8 transformer blocks and a context length of 512 tokens, varying model size through the width, trained with the Adam optimizer ($\beta_1{=}0.9$, $\beta_2{=}0.95$, no weight decay) under a constant learning rate schedule with linear warmup. We tune the base learning rate on a small model and transfer it to larger models using $\mu$P \citep{yang2022tensor}, arriving at a base learning rate of 2 for all experiments. In $\mu$P, the per-layer learning rate is the base learning rate divided by the input dimension, so our reported base learning rate is larger than typical learning rates used for Adam. Width, depth, learning rate, batch size, and initialization seed are shared between the student and teacher, so the per-step KL starts at zero.

\subsection{Code Length Computation}\label{app:codelen}
For requential coding, we report the code length bound $\widehat L_\mathrm{req}$ from Eq.~\eqref{eq:req_codelen}. Numerically, the per-step KL between the (EMA) teacher and the (EMA) student $\mathrm{KL}(Q_t\|P_t)$ is estimated by Monte Carlo on the synthetic batch $X_t$ that the student trains on: since $X_t \sim Q_t$, $\log Q_t(X_t)-\log P_t(X_t)$ is an unbiased estimate of $\mathrm{KL}(Q_t\|P_t)$. Our batch size exceeds $10^5$ tokens in all experiments, making this estimator quite accurate in practice. The lower-order universal-code terms in Eq.~\eqref{eq:req_codelen} are included explicitly in the block-size tradeoff calculation and omitted elsewhere since our large batch size makes these terms completely negligible (see Figure~\ref{fig:block-size-tradeoff}). For prequential coding, the code length is the cumulative cross-entropy loss of a model trained on real data, $L_{\mathrm{preq}} = \sum_{t=1}^{T} -\log P_{t-1}(X_t).$ For this model, we reuse the teacher in requential coding (with iso-loss projection turned off).

\subsection{Datasets}\label{app:owt}\label{app:c5m}
We use the OpenWebText dataset at \url{https://huggingface.co/datasets/Skylion007/openwebtext}, keeping only documents restricted to 96 common alphanumeric symbols, and apply character-level tokenization with vocabulary size $V = 96$. We use the CIFAR-5M dataset \citep{nakkiran2020deep} at \url{https://github.com/preetum/cifar5m}, converting the $32 \times 32 \times 3$ images to greyscale and flattening to a 1D sequence of 1024 tokens in raster-scan order, with the pixel intensities $\{0,\ldots,255\}$ as the vocabulary ($V = 256$). We use the FineWeb dataset \citep{penedo2024fineweb} with the GPT-2 BPE tokenizer ($V = 50257$). All experiments use a sequence length of 512 tokens.

\subsection{EMA, Teacher Smoothing and Iso-Loss Projection}\label{app:smoothing-projection}
Algorithm~\ref{alg:smoothing-projection} gives the general requential encoder used in our experiments. It maintains a raw teacher that trains on the real batches and a raw student that trains on the decoded synthetic batches, together with exponential moving averages of both, and applies iso-loss projection according to a pre-defined schedule.  We now describe these two techniques in detail.

\begin{algorithm}[t]
\caption{\small\textbf{Requential encoding with teacher smoothing and iso-loss projection.} The raw teacher $Q$ trains on the real batches $Z_0, Z_1, \ldots$ and the raw student $P$ on the decoded synthetic batches. We initialize the teacher from the same point as the student and explicitly show its updates. We keep exponential moving averages $Q^{\text{EMA}}$ and $P^{\text{EMA}},$ used as the REC target and reference. The boolean $\mathrm{proj}$ toggles iso-loss projection at the steps $\mathcal{S}$.  $Q$ and $P$ are understood to be bundled with their optimizer states (the Adam moments), so the assignment copies the full training state of the student into the teacher's upon each reset. }
\label{alg:smoothing-projection}
\footnotesize
\begin{algorithmic}[1]
\REQUIRE real data batches $Z_0, Z_1, \ldots$, \; boolean $\mathrm{proj}$, \; projection steps $\mathcal{S}$, \; student initialization $P_0$, \; update rule $G$, \; PRNG seed $s$
\ENSURE messages $(m_t)_{t=0}^{T-1}$
\STATE initialize raw teacher $Q_0 \leftarrow P_0$, \quad $Q^{\text{EMA}}_0 \leftarrow P_0$, \; $P^{\text{EMA}}_0 \leftarrow P_0$, \; $j \leftarrow 0$ \hfill {\color{black!55}// raw student starts at $P_0$, \, $j$ points to the next real batch}
\FOR{$t = 0$ to $T-1$}
  \IF{$\mathrm{proj}$ \textbf{and} $t \in \mathcal{S}$}
    \STATE $\ell^\ast \leftarrow$ validation loss of $Q^{\text{EMA}}_t$ \hfill {\color{black!55}// pre-projection loss}
    \STATE reset teacher to student: $Q_t \leftarrow P_t$, \; $Q^{\text{EMA}}_t \leftarrow P^{\text{EMA}}_t$
    \WHILE{validation loss of $Q^{\text{EMA}}_t$ exceeds $\ell^\ast$}
      \STATE $Q_t \leftarrow G(Q_t, Z_j)$, \; $j \leftarrow j+1$, \; $Q^{\text{EMA}}_t \leftarrow \mathrm{EMA}(Q^{\text{EMA}}_t, Q_t)$ \hfill {\color{black!55}// teacher recovery, student paused: no bits}
    \ENDWHILE
  \ENDIF
  \STATE $S_t \leftarrow \mathrm{PRNG}(s,t,\cdot)$
  \STATE $m_t \leftarrow \mathrm{REC.Encode}(Q^{\text{EMA}}_t, P^{\text{EMA}}_t, S_t)$, \quad $X_t \leftarrow \mathrm{REC.Decode}(P^{\text{EMA}}_t, m_t, S_t)$ \hfill {\color{black!55}// $X_t \sim Q^{\text{EMA}}_t$}
  \STATE $P_{t+1} \leftarrow G(P_t, X_t)$, \quad $Q_{t+1} \leftarrow G(Q_t, Z_j)$, \; $j \leftarrow j+1$ \hfill {\color{black!55}// raw models train}
  \STATE $P^{\text{EMA}}_{t+1} \leftarrow \mathrm{EMA}(P^{\text{EMA}}_t, P_{t+1})$, \quad $Q^{\text{EMA}}_{t+1} \leftarrow \mathrm{EMA}(Q^{\text{EMA}}_t, Q_{t+1})$
\ENDFOR
\STATE \textbf{return} $(m_t)_{t=0}^{T-1}$ \hfill {\color{black!55}//  length $\approx \sum_{t} \mathrm{KL}(Q^{\text{EMA}}_t\,\|\,P^{\text{EMA}}_t)$}
\end{algorithmic}
\end{algorithm}

\paragraph{EMA and Teacher Smoothing.}
In place of learning rate decay, we apply exponential moving average (EMA) of the iterates, which has a similar noise-reduction effect \citep{hagele2024scaling}. For the student, the EMA state is updated as $P^{\text{EMA}}_{t} = e^{-\frac{1}{\alpha t}} P^{\text{EMA}}_{t-1} + \qty(1 - e^{-\frac{1}{\alpha t}}) P_t,$ where $(P_t)_t$ are the raw student iterates, and $\alpha$ sets the EMA timescale as a fraction of current elapsed steps. We always apply EMA to the student with $\alpha=0.01$, and use the EMA student to evaluate the loss on real data and the teacher-student KL.  

When teacher smoothing is enabled, we also apply EMA to the teacher and use the EMA teacher to generate training data for the student, evaluate the teacher-student KL and compute teacher's loss on real data. The teacher EMA state is updated as $Q^{\text{EMA}}_{t} = e^{-\frac{1}{\max(50, \alpha t)}} Q^{\text{EMA}}_{t-1} + \qty(1 - e^{-\frac{1}{\max(50, \alpha t)}}) Q_t,$ where $(Q_t)_t$ are the raw teacher iterates and the decay timescale is clipped from below at $50$ steps to avoid under-smoothing the teacher early on in training, which we found led to very large initial KL. Algorithm~\ref{alg:smoothing-projection} abbreviates either update as $\mathrm{EMA}(\cdot,\cdot)$, mixing the current moving average (first argument) with the latest raw iterate (second argument) to produce the next moving average.

\paragraph{Iso-loss projection.}
At a given set of projection steps $\mathcal{S}$, 
\emph{iso-loss projection} resets the teacher state to the student state and then performs a teacher recovery phase. The state includes the model weights and optimizer buffers, so the reset hands the teacher a fully identical training state to the student's before further training. During the teacher recovery phase, the teacher trains on real data with the student temporarily paused until the teacher's pre-projection loss is recovered. The student then resumes training together with the teacher until the next projection step is encountered. 
The projection steps $\mathcal{S}$ are geometrically spaced, the first at the 100-th student step and each subsequent step is $1.5\times$ the previous (i.e., $\mathcal{S} = \{100, 150, 225, \ldots\})$, matching the fact that meaningful progress occurs on a logarithmic timescale for training on natural data \citep{kaplan2020scaling}. No synthetic tokens are generated during a recovery, so it adds teacher compute and data but no student code length.

\paragraph{Decoding.} The decoder runs \textsc{Decode} from Figure~\ref{alg:requential-algs} except using the EMA student as the REC reference, and needs none of the teacher-side machinery of Algorithm~\ref{alg:smoothing-projection}. Reconstructing the student trajectory step by step, it forms each $P^{\text{EMA}}_t$ from the messages and the shared randomness, decodes $X_t$, and applies the same student update $P \leftarrow G(P, X_t)$ followed by the EMA update. Teacher smoothing and iso-loss projection are entirely invisible at the decoder.

\paragraph{Evaluating Code Length.} As before, we never run REC to measure code length. Instead, we replace the REC encode-decode pair with the equivalent process that draws each $X_t$ directly from the EMA teacher, and report the cumulative conditional KL $\sum_t \mathrm{KL}\!\left(Q^{\text{EMA}}_t \,\|\, P^{\text{EMA}}_t\right)$ between the EMA teacher and EMA student over the coded steps. With teacher smoothing disabled, $Q^{\text{EMA}}_t$ is replaced by the raw teacher $Q_t.$ Teacher recovery during iso-loss projection generates no synthetic batches and contributes zero bits.

\subsection{Comparison with Prequential Coding and PTQ (Figure~\ref{fig:train_curves})}
We train models with roughly 100M parameters (including embedding) on each dataset, with a batch size of 1024 sequences (0.5M tokens per step) and a learning rate warmup of 131M tokens. The teacher trains on up to 5B real tokens on OpenWebText and CIFAR-5M and 20B on FineWeb, and the figure truncates every curve at the Chinchilla budget of $D = 20N$ student tokens so every coded student model can be compared after training on the same number of tokens. The vanilla requential coding did not use teacher smoothing or iso-loss projection, but kept the student EMA.

\subsection{Model Size Scaling at Fixed Data (Figure~\ref{fig:scaling-model})}
On OpenWebText and CIFAR-5M we sweep widths 128 to 1600 (1.7M to 247M parameters), training each model on 5B real tokens with a batch size of 256 sequences and a warmup of 16M tokens. On FineWeb we sweep widths 128 to 2752 (14M to 1B parameters), training on 20B real tokens with a batch size of 1024 sequences, and drop the smallest model from the figure. All runs use iso-loss projection and teacher smoothing. The target loss in the bottom row is the final loss of the smallest plotted model.

\subsection{Ensembles (Figure~\ref{fig:ensemble})}
We train ensembles of $E \in \{1, 2, 4, 8\}$ members on FineWeb, each member a width-512 transformer with 76.9M parameters, sharing a single teacher of the same architecture, with a batch size of 1024 sequences and the per-member Chinchilla budget of 1.54B real tokens. These runs use teacher smoothing, but for simplicity we did not use iso-loss projection since the choice which student to project is ambiguous. All members train on the same synthetic batch from different initialization seeds, and the per-step KL is computed between the teacher and the average of the members' predictions, so a single message stream reconstructs the entire ensemble at the decoder.

\subsection{Generalization Bounds (Figures~\ref{fig:generalization} and~\ref{fig:generalization-owt})}\label{app:bounds}
We evaluate the bound of Theorem~\ref{thm:finzi_bound_rewritten}, with $K=$ \texttt{np.linspace(0.0, 1.0, 1002)[1:-1]} ($|K| = 1000$) and failure probability $\delta = 0.01$. For the requential bound, we use teacher smoothing and iso-loss projection, set $D$ as the number of real tokens processed by the teacher, and estimate the empirical risk and the statistic $\Sigma$ of the EMA student on the teacher training data. The PTQ bounds instead apply to a model trained normally on $D$ real tokens without iso-loss projection, coding its weights at $4$ bits per parameter ($L = 4N$) and using its own empirical risk and statistic $\Sigma$ without actually performing quantization, representing an idealized lossless PTQ. The fixed-data column of Figure~\ref{fig:generalization} terminates each run at $D = 2$B real tokens, the compute-optimal column evaluates at $D = 20N$, and the bits-per-parameter column reports $\widehat L_\mathrm{req}/N$ at $D = 20N$ with a power-law fit to the five largest models. Figure~\ref{fig:generalization-owt} reports the corresponding OpenWebText bound panels, which we moved out of Figure~\ref{fig:generalization} for space.

\begin{figure}[t]
\centering
\includegraphics[width=0.78\linewidth]{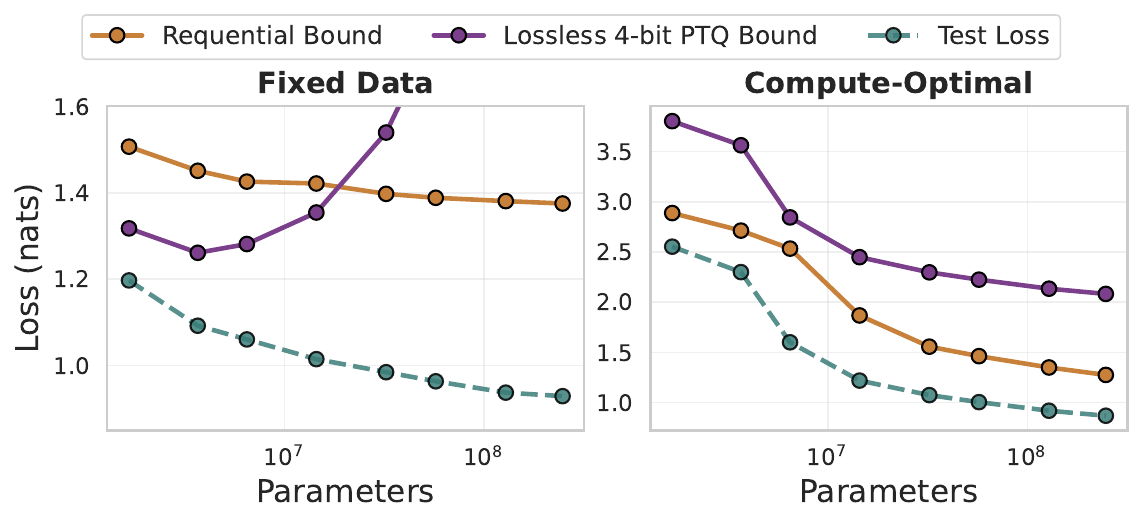}
\caption{\small\textbf{OpenWebText generalization bounds follow the same trends as CIFAR-5M and FineWeb in Figure~\ref{fig:generalization}.} The requential bound improves with scale under both the fixed-data protocol (\textbf{left}, $D{=}2$B tokens, falling from $1.51$ to $1.38$ nats as $N$ grows from 1.7M to 247M) and the Chinchilla protocol (\textbf{right}, $D{=}20N$), while the PTQ bounds on a vanilla teacher diverge. The plotted teacher loss is the test loss of the vanilla teacher on which the PTQ bounds are built. The certified gap at the Chinchilla budget is down to $0.26$ nats by 247M parameters.}
\label{fig:generalization-owt}
\end{figure}

\section{Generalization Bound}
\label{appx:discuss_bound}

For reference, we restate the generalization bound in \citet{finzi2025compute}. The bound relates the population tokenwise risk of a learned predictor to its empirical risk and a complexity term derived from a valid prefix-free code of the model. All complexity-dependent terms in the bound are monotone in $C$, so any improvement in the description length directly tightens the resulting guarantee.

\begin{theorem}[Generalization bound (adapted from \citet{finzi2025compute})]
\label{thm:finzi_bound_rewritten}
Let $X_{1:D}$ be a (possibly dependent) sequence of tokens in a vocabulary of size $V$.
For any model $h$, let $p_h(\cdot \mid X_{<k})$ denote its predictive distribution and define the tokenwise negative log-likelihood (in nats)
\[
R_h(X_k \mid X_{<k}) := - \ln p_h(X_k \mid X_{<k}).
\]
Define the empirical risk and tokenwise expected (population) risk
\[
\widehat R_h := \frac{1}{D}\sum_{k=1}^D R_h(X_k \mid X_{<k}),
\qquad
R_h := \frac{1}{D}\sum_{k=1}^D \mathbb{E}\!\left[R_h(X_k \mid X_{<k}) \mid X_{<k}\right].
\]

Fix a finite set $K \subset (0,1)$ and a confidence level $\delta \in (0,1)$.
Assume there exists $\Delta > 0$ such that for all $k$,
\begin{equation}
\label{eq:finzi_delta_assump}
R_h(X_k \mid X_{<k}) \;-\; \mathbb{E}_{Y_k \sim p_h(\cdot \mid X_{<k})}\!\left[ R_h(Y_k \mid X_{<k}) \right] \;\le\; \Delta.
\end{equation}
Let $L(h)$ be the length of a prefix-free code for $h$ in bits, and define the per-token complexity
\[
C := \frac{L(h)\ln 2 + \ln\frac{|K|}{\delta}}{D}.
\]

Define the normalized deviations
\[
A_k := \frac{\mathbb{E}_{Y_k \sim p_h(\cdot \mid X_{<k})}[ R_h(Y_k \mid X_{<k})] - R_h(X_k \mid X_{<k})}{\Delta},
\]
the function $v(a) := a - \ln(1+a)$ for $a > -1$, and
\[
\Sigma(C,\Delta,\{A_k\}_{k=1}^D,K)
:=
\min_{s \in K}
\left\{
\Delta \sqrt{C}\,\frac{1-s}{s}
\;+\;
\frac{\Delta}{\sqrt{C}}\cdot \frac{1}{D}\sum_{k=1}^D \frac{v(sA_k)}{s}
\right\}.
\]

Then, simultaneously for all $h$ in the coded hypothesis class, with probability at least $1-\delta$,
\begin{equation}
\label{eq:finzi_core_bound_rewritten}
R_h \;\le\; \widehat R_h \;+\; \Delta C \;+\; \Sigma(C,\Delta,\{A_k\}_{k=1}^D,K)\sqrt{C}.
\end{equation}

Moreover, for the categorical NLL, there exists a prediction-smoothed model
$p_{s,h}(\cdot \mid X_{<k}) = (1-\alpha)p_h(\cdot \mid X_{<k}) + \alpha/V$
for some $\alpha \in (0,1)$ such that
\begin{equation}
\label{eq:finzi_smoothing_step_rewritten}
R_{s,h} \;\le\; \widehat R_h \;+\; C\ln V
\;+\; \Sigma(C,\Delta,\{A_k\}_{k=1}^D,K)\sqrt{C}
\;+\; \sqrt{2C}.
\end{equation}
\end{theorem}

\end{document}